\documentclass[final,12pt]{colt2023} 

\title[Projection-free Online Exp-concave Optimization]{Projection-free Online Exp-concave Optimization}
\usepackage{times}

\newtheorem{observation} {Observation}
\newtheorem{assumption} {Assumption}

\DeclareMathOperator*{\argmin}{argmin}

\newcommand{\brac}[1]{\left(#1\right)}
\newcommand{\enorm}[1]{\left\Vert#1\right\Vert}
\newcommand{\matnorm}[2]{\left\Vert#1\right\Vert_{#2}}

\def\vz{{\textbf{0}}}
\def\g{{\mathbf{g}}}
\def\m{{\mathbf{m}}}

\def\x{{\mathbf{x}}}

\def\a{{\mathbf{a}}}
\def\u{{\mathbf{u}}}
\def\v{{\mathbf{v}}}
\def\z{{\mathbf{z}}}

\def\y{{\mathbf{y}}}

\def\A{{\mathbf{A}}}

\def\I{{\mathbf{I}}}
\def\B{{\mathbf{B}}}
\def\C{{\mathbf{C}}}
\def\V{{\mathbf{V}}}

\def\H{{\mathbf{H}}}
\def\S{{\mathbf{S}}}

\newcommand{\R}{\mathcal{R}}

\newcommand{\mC}{\mathcal{C}}
\newcommand{\mK}{\mathcal{K}}

\newcommand{\ball}{\mathcal{B}}
\newcommand{\mbS}{\mathbb{S}}

\newcommand{\dist}{\textrm{dist}}
\newcommand{\diag}{\textbf{diag}}

\newcommand{\trace}{\textrm{Tr}}
\newcommand{\rank}{\textrm{rank}}
\newcommand{\reals}{\mathbb{R}}

\newcommand{\RNum}[1]{\uppercase\expandafter{\romannumeral #1\relax}}

\usepackage{makecell}
\usepackage{multirow}

\usepackage{mathtools}
\usepackage{amssymb}
\usepackage{amsmath}
\usepackage{soul}
\usepackage{xcolor}

\usepackage{hyperref}

\usepackage{algpseudocode}
\usepackage[ruled]{algorithm2e}
\SetKwRepeat{Do}{do}{while}

\algnewcommand{\comment}[1]{\Comment{{#1}}}

\coltauthor{%
 \Name{Dan Garber} \Email{dangar@technion.ac.il}\\
 \addr Technion - Israel Institute of Technology
 \AND
 \Name{Ben Kretzu} \Email{benkretzu@campus.technion.ac.il}\\
 \addr Technion - Israel Institute of Technology
}

\begin{document}

\maketitle

\begin{abstract}%
We consider the setting of online convex optimization (OCO) with \textit{exp-concave} losses. The best regret bound known for this setting is $O(n\log{}T)$, where $n$ is the dimension and $T$ is the number of prediction rounds (treating all other quantities as constants and assuming $T$ is sufficiently large), and is attainable via the well-known Online Newton Step algorithm (ONS). However, ONS requires on each iteration to compute a projection (according to some matrix-induced norm) onto the feasible convex set, which is often computationally prohibitive in high-dimensional settings and when the feasible set admits a non-trivial structure. In this work we consider projection-free online algorithms for  exp-concave and smooth losses, where by projection-free we refer to algorithms that rely only on the availability of a linear optimization oracle (LOO) for the feasible set, which in many applications of interest admits much more efficient implementations than a projection oracle. We present an LOO-based ONS-style algorithm, which using overall $O(T)$ calls to a LOO, guarantees in worst case regret bounded by $\widetilde{O}(n^{2/3}T^{2/3})$ (ignoring all quantities except for $n,T$). However, our algorithm is most interesting in an important and plausible low-dimensional data scenario: if the gradients  (approximately) span a subspace of dimension at most $\rho$, $\rho << n$, the regret bound improves to $\widetilde{O}(\rho^{2/3}T^{2/3})$, and by applying standard deterministic sketching techniques, both the space and average additional per-iteration runtime requirements are only $O(\rho{}n)$ (instead of $O(n^2)$). This  improves upon recently proposed LOO-based algorithms for OCO which, while having the same state-of-the-art dependence on the horizon $T$, suffer from regret/oracle complexity that scales with $\sqrt{n}$ or worse.
\end{abstract}


\section{Introduction} \label{sec:intro}

This work contributes to the line of research on efficient projection-free algorithms for online convex optimization (OCO), which has received a significant amount of interest in the theoretical machine learning community in recent years, see for instance \cite{Hazan12, garber2013playing, chen2019projection, garber2020improved, kretzu2021revisiting, hazan2020faster, Levy19, wan2021projection, ene2021projection, pmlr-v80-chen18c, zhang2017projection, Garber22a, mhammedi2021efficient, mhammedi2022, lu2022projection}. We recall that in the setting of OCO \cite{HazanBook, Shalev12} (see formal definition in Section \ref{sec:setup}) a decision maker is required  to iteratively choose an action --- a point in some convex and compact set $\mK\subset\reals^n$ (fixed throughout all iterations) \footnote{for ease of presentation we consider the underlying vector space to be $\reals^n$, however any finite Euclidean space will work}, where after her selection, a convex loss function from $\mK$ to $\reals$ is revealed and the decision maker incurs a loss which equals the value of the loss function evaluated at the point chosen on that round. The performance of the decision maker is measured via the standard notion of regret which is the difference between her accumulated loss throughout all $T$ rounds (where $T$ here is assumed to be known in advanced) and that of the best fixed point in $\mK$ in hindsight. Throughout this work we consider the full-information setting, where after each round, the decision maker gains full knowledge of the loss function used on that round. The term \textit{projection-free} refers to algorithms which avoid the computation of orthogonal projections onto the feasible set $\mK$, as required by most standard algorithms, and instead only access the feasible set via conceptually simpler computational primitives, such as an oracle for linear optimization over $\mK$ (LOO). The motivation for such methods is that indeed for many feasible sets of interest and for high-dimensional problems, implementing the LOO can be much more efficient than projection, see for instance detailed examples in \cite{Jaggi13, Hazan12} (see in the sequel discussion on other projection-free oracles). 
 
In this work we consider, to the best of our knowledge for the first time, efficient projection-free LOO-based algorithms for OCO in case all loss functions are \textit{exp-concave}. We recall that a function $f(\x)$ is  $\alpha$-exp-concave for some $\alpha >0$, if  the function $e^{-\alpha{}f(\x)}$ is concave \footnote{in linear regression one has $g(x) = (x-b)^2$ for some $b\in\reals$, and in online portfolio selection one has $g(x) = -\log(x)$, which is strongly convex on $\reals_{>0}$}. Exp-concavity is a property which is well known to allow for faster convergence rates (in terms of regret). In particular, exp-concave losses underly some of the most important applications of OCO such as online linear regression and  online portfolio selection. More generally, any loss of the form $f(\x) = g(\a^{\top}\x)$ with $g:\reals\rightarrow\reals$ strongly convex, is exp-concave.
While for general convex functions the optimal regret bound attainable (by any algorithm) is $O(\sqrt{T})$ (treating all quantities except for $n,T$ as constants), in case all losses are exp-concave, a regret bound of the form $O(n\log{}T) $ is attainable \cite{hazan2007logarithmic, HazanBook}, which is faster for any fixed dimension $n$ and $T$ large enough. The latter regret bound is attainable via a well-known algorithm known as \textit{Online Newton Step} (ONS), however, ONS requires on each iteration to compute a non-Euclidean projection onto the feasible set w.r.t. to some matrix-induced norm (this matrix aggregates all the gradients of the losses observed so far), and hence is often computationally prohibitive in high-dimensional settings and when the feasible set $\mK$ admits non-trivial structure. 

Our main contribution is a novel projection-free LOO-based variant of ONS for exp-concave and smooth (Lipschitz continuous gradient) losses. Using overall $O(T)$  calls to a LOO (throughout all rounds), our algorithm guarantees in worst case $\widetilde{O}(n^{2/3}T^{2/3})$ regret (where currently for ease of presentation we treat all quantitates except for $n,T$ as constants, and $\widetilde{O}$ hides poly-logarithmic factors). However, our algorithm is most interesting in  a highly popular and plausible scenario in high-dimensional analytics, namely when the observed gradients of the loss functions (the data fed into the algorithm), approximately, span only a low dimensional subspace. Denoting by $\rho$ the (approximate) dimension of the subspace spanned by the gradients,  by a simple tuning of parameters, our regret bound improves to $\widetilde{O}(\rho^{2/3}T^{2/3})$, which is independent of the ambient dimension $n$. Moreover, by leveraging well-known efficient deterministic sketching techniques, as was already proposed in \cite{luo2016efficient} (but not in the context of projection-free algorithms), we can also reduce the memory and additional average runtime per iteration from $O(n^2)$ to only $O(\rho{}n)$, i.e., linear in the dimension for a constant $\rho$. 

To put our results in perspective, the best previous regret bound for a LOO-based algorithm for OCO (that holds for arbitrary convex losses and with no assumption on the span of the gradients) which is dimension-independent is $O(T^{3/4})$ and requires overall  $O(T)$ calls to the LOO \cite{Hazan12, Garber22a}. Two recently proposed LOO-based algorithms also improved the dependence on the horizon $T$ from $T^{3/4}$ to $T^{2/3}$, however suffer from regret and/or oracle complexity which scales with $\sqrt{n}$ or worse: the regret bound of the Follow The Perturbed Leader-based algorithm of  \cite{hazan2020faster} has a regret bound of the form $O(\sqrt{n}T^{2/3}$), while the Follow The Leader-based algorithm of \cite{mhammedi2022} (which is based on approximating the feasible set with a strongly convex set, which leads to the faster rate) requires overall $\widetilde{O}(nT)$ calls to the LOO and has a regret bound of the form $O((R/r)^{2/3}T^{2/3})$, where $R/r$ is the ratio between an enclosing ball and an enclosed ball, which often scales with $\sqrt{n}$ and even with $n$ (e.g., for the simplex or the spectrahedron, see \cite{mhammedi2021efficient}). Moreover, the additional runtime per iteration of the algorithm in \cite{mhammedi2022} scales with $n^3$.
Unfortunately, such explicit dependencies on the ambient dimension may be prohibitive for high-dimensional problems, which is indeed the typical setting of interest for projection-free methods. Thus, it is interesting whether it is possible to obtain a fast $T^{2/3}$ rate  without explicit dependence on the ambient dimension. 

A very popular approach to circumvent explicit dependencies on the ambient dimension, which underlies numerous models in statistics/high-dimensional analytics and is observed frequently in real-world scenarios, is the assumption that the data, at least approximately, lies  only in a low-dimensional subspace. In our context of OCO with exp-concave losses, as discussed above, many losses of interest take the form $f_t(\x) = g_t(\a_t^{\top}\x)$,  $g_t:\reals\rightarrow\reals$, with the gradient vector being $\nabla{}f_t(\x) = g_t'(\a_t^{\top}\x_t)\a_t$. Thus, when the data vectors $\a_1,\dots,\a_T$ approximately span only a low-dimensional subspace (in the sense that the eigenvalues of the unnormalized covariance $\lambda_i(\sum_{t=1}^T\a_t\a_t^{\top})$ are sufficiently small for all $i\geq \rho+1$, for some $\rho<<n$), our regret bound becomes dimension-independent and thus suitable for such popular high-dimensional settings. To the best of our knowledge, the fast (in terms of $T$, but not $n$) algorithms proposed in  \cite{hazan2020faster, mhammedi2022} cannot efficiently leverage low-dimensionality of the gradients.

Table \ref{table:Op} gives a short summery of our results as well as a comparison to related LOO-based algorithms for OCO.

On the technical side, our work primarily builds on the recent approach of \cite{Garber22a} which suggested a LOO-based projection-free variant of the well known Euclidean Online Projected Gradient Descent method \cite{Zinkevich03, HazanBook}, and is based on the concept of \textit{approximately-feasible (Euclidean) projections} \footnote{\cite{Garber22a} originally used the terminology \textit{close infeasible projections}}, which refers to the computation of points which on one-hand, while infeasible w.r.t. the decision set $\mK$, still  satisfy certain properties related to orthogonal projections and are sufficiently close to the feasible set, which drives the regret bound, and on the other-hand, could be computed efficiently using only a limited number of queries to the LOO of the feasible set via the classical Frank-Wolfe algorithm for \textit{offline} convex minimization \cite{FrankWolfe, Jaggi13}. Here we provide a non-trivial extension of  this framework, from supporting only Euclidean (approximately feasible) projections, to supporting projections w.r.t. matrix-induced norms as employed by ONS. We also substantially improve the bound on the oracle complexity required to compute such approximately-feasible projections, which is crucial to obtaining our faster regret rate ($T^{2/3}$ instead of $T^{3/4}$ in  \cite{Garber22a}).

\paragraph{Other projection-free oracles:}
We mention in passing that while, as in this work, most literature on projection-free OCO  assumes the feasible set is accessible through a LOO, some recent works have also considered other oracles such as a separation oracle or a membership oracle \cite{mhammedi2021efficient, Garber22a, lu2022projection}. While each of these oracles could be implemented via the others (see for instance \cite{tat2017efficient}), none of them is generically superior to the other (in terms of efficiency of implementation). Finally, the very recent work \cite{mhammedi2022oqns} considers an efficient variant of ONS which is based on accessing the feasible set only through a separation oracle, however it requires the feasible set $\mK$ to by symmetric in the sense that $\mK = -\mK$, which is fairly restrictive.

\begin{table} \renewcommand{\arraystretch}{1.4}
{\footnotesize
\begin{center}
  \begin{tabular}{| c | c  | c | c | c | c | c |} \hline
    Reference & \makecell{Additional \\ assumptions} & \makecell{Based \\ on} & \makecell{Deter-\\ministic?} & \makecell{LOO \\ calls }& \makecell{Additional \\ runtime}  &  Regret  \\ \hline
     \makecell*[{{p{2.1cm}}}]{\cite{Hazan12}} & - & RFTL & \checkmark &  $T$  &  $n T $  &  $ RG  T^\frac{3}{4}$ \\ \hline
    \makecell*[{{p{2.1cm}}}]{\cite{Garber22a}} & -  & OGD & \checkmark &  $T$  &  $n T $  & $RG  T^\frac{3}{4}$  \\ \hline
    \makecell*[{{p{2.1cm}}}]{\cite{mhammedi2022}} & $r \ball \subseteq \mK $ &  FTL & $\times$ &  $nT$  &  $n^3 T $  & $ GR(R/r)^{\frac{2}{3}} T^\frac{2}{3}$  \\ \hline
     \makecell*[{{p{2.1cm}}}]{\cite{hazan2020faster}}  & \makecell{ $\beta$-smooth \\ losses }  & FTPL & $\times$ &  $T$  &  $n T $  & $R\brac{G\sqrt{n}+\beta R} T^\frac{2}{3}$ \\ \hline
    Theorem \ref{thm:mainthm:short} & \makecell{ $\alpha$-exp concave \\ and $\beta$-smooth losses } & ONS & \checkmark & $T$ & $n^2T$ &    \makecell{$\brac{G  + \alpha^{-1} }R n^\frac{2}{3} T^\frac{2}{3}$ \\ $ + \beta R^2 T^\frac{2}{3}$} \\ \hline
    Theorem \ref{thm:LOO-ONS-FDS} & \makecell{ 1. $\alpha$-exp concave \\ and $\beta$-smooth losses \\
    2. gradients approx.\\ span  $\rho$-dim. subspace (*)}& ONS & \checkmark  &$T$  & $\rho nT $ &  \makecell{$\brac{G  + \alpha^{-1} }R\rho^\frac{2}{3} T^\frac{2}{3}$ \\ $ + \beta R^2 T^\frac{2}{3}$}  \\ \hline 
  \end{tabular} 
\caption{ Summary of results and comparison to previous LOO-based methods (applicable to arbitrary convex and compact sets). 
$G$ denotes an upper-bound on the $\ell_2$ norm of the gradients, $R$ denotes the radius of the feasible set $\mK$, and $\ball$ denotes the unit Euclidean ball centered at the origin. Condition (*) should be understood as $\sum_{i=\rho+1}^n\left({\sum_{t=1}^T\nabla_t\nabla_t^{\top}}\right) = O(T^{2/3})$, where $\lambda_i(\cdot)$ denotes the $i$-th largest eigenvalue and $\nabla_t\in\reals^n$ denotes the gradient of the loss observed on round $t$.
The bounds omit constants and poly-logarithmic factors.} \label{table:Op}
\end{center}
}
\end{table}\renewcommand{\arraystretch}{1}

\section{Preliminaries}

\subsection{Notation}
We let $\Vert{\cdot}\Vert$ denote the Euclidean norm over $\reals^n$. For a positive semidefinite matrix $\A$, we let $\Vert{\cdot}\Vert_{\A}$ denote the induced norm over $\reals^n$, i.e., for any $\x\in\reals^n$, $\Vert{\x}\Vert_{\A} = \sqrt{\x^{\top}\A\x}$. We let $\mbS^n, \mbS^n_+, \mbS^n_{++}$ denote the space of real symmetric $n\times n$ matrices, the set of all real $n\times n$ (symmetric) positive semidefinite matrices, and the set of all real $n\times n$ (symmetric) positive definite matrices, respectively. We use the standard notation $\A\succeq 0$ ($\A\succ 0$) to denote that $\A\in\mbS^n_{+}$ ($\A\in\mbS^n_{++}$). For a matrix $\A\in\mbS^n$ and $i\in[n]$, we let $\lambda_i(\A)$ denote the $i$-th largest (signed) eigenvalue of $\A$. We denote by $\A \bullet \B$ the standard inner product between two matrices in $\mbS^n$, i.e., $\A \bullet \B  = \sum_{i=1}^{n} \sum_{j=1}^{n} \A_{i,j}\B_{i,j} = \trace(\A\B^\top)$. We let $\ball$ denote the unit Euclidean ball in $\reals^n$ centered at the origin. Given a convex and compact set $\mC\subset\reals^n$, a point $\x\in\reals^n$, and a positive definite matrix $\A\in\mbS^n_{++}$, we let $\dist(\x,\mC)$ and $\dist_{\A}(\x,\mC)$ denote the Euclidean distance of $\x$ from $\mC$ and the distance induced by $\A$ of $\x$ from $\mC$, respectively. That is, $\dist(\x,\mC) = \min_{\y\in\mC}\Vert{\x-\y}\Vert$, $\dist(\x,\mC) = \min_{\y\in\mC}\Vert{\x-\y}\Vert_{\A}$. 

\subsection{Problem setup: online exp-concave and smooth optimization with a LOO}\label{sec:setup}
We recall the setting of OCO \cite{HazanBook, Shalev12}, in which, a decision maker is required for $T$ rounds ($T$ is assumed known in advanced for simplicity), to select on each round some point $\x^t\in\mK$, where $\mK\subset\reals^n$ is convex and compact (and fixed throughout all rounds). After making her choice on round $t$, the decision maker observes a convex loss function $f_t:\mK\rightarrow\reals$ and incurs the loss $f_t(\x^t)$. The goal of the decision maker is to minimize her  regret which is given by
\vspace{-15pt}
\begin{align*}
    \mathcal{R}_T = \sum_{t=1}^{T} f_t(\x^t) - \min_{\x \in \mK} \sum_{t=1}^{T} f_t(\x),
\end{align*}
i.e., it is the difference between her cumulative loss, and the cumulative loss of the best fixed point in $\mK$ in hindsight.

Throughout this work we assume the feasible set is accessible through a linear optimization oracle, which means that for any $\g\in\reals^n$ we can efficiently compute some $\v^*\in\argmin_{\v\in\mK}\v^{\top}\g$. 

We now turn to discuss our specific assumptions on the loss functions $f_1,\dots,f_T$. 
In the following definitions we let $\mC$ denote a convex and compact subset of $\reals^n$.
\begin{definition}\label{def:smooth}
    We say $f: \mC \to \reals$ is $\beta$-smooth over $\mC$, for some $\beta \geq 0$, if for every $\x,\y \in \mC$ it holds that $\enorm{\nabla f(\x) - \nabla f(\x)} \leq \beta \enorm{\x-\y}$.
\end{definition} 
\begin{definition}\label{def:exp_concave}
   We say $f: \mC \to \reals$ is $\alpha$-exp concave over $\mC$, for some $\alpha > 0$, if $e^{-\alpha f(\x)}$ is concave over $\mC$.
\end{definition} 
We recall that an exp-concave function is in particular convex (see \cite{HazanBook}). In fact, we shall consider a weaker condition than exp concavity, which we shall refer to as a curvature condition.
\begin{definition}\label{def:exp_concave_property}
    Let $R$ denote the radius of $\mC$, i.e., $\max_{\x,\y\in\mC}\Vert{\x-\y}\Vert \leq 2R$. A differentiable function $f:\mC\rightarrow{}R$ with gradients upper-bounded in $\ell_2$ norm by some $G>0$ over $\mC$, is said to satisfy the \textit{curvature condition} over $\mC$ with some parameter $\alpha>0$, if for every $\eta \geq \max\{ 4GR, 2/\alpha \}$ and  every $\x,\y \in \mC$, it holds that
        $f(\x) - f(\y)  \leq \nabla f(\x)^\top \brac{\x -\y} -  \frac{1}{2 \eta} \brac{\x -\y}^\top \nabla f(\x) \nabla f(\x)^\top \brac{\x -\y}$.
\end{definition}
This condition is weaker than exp-concavity in the sense that an $\alpha$-exp-convave function also satisfies the curvature condition with the same parameter $\alpha$ \cite{HazanBook}.

The following assumption records all of our assumptions on the loss functions $f_1,\dots,f_T$, which we assume to hold throughout the rest of the paper.
\begin{assumption}\label{ass:mainass}
The loss functions $f_1,\dots,f_T$, are all  $\beta$-smooth, have gradients upper-bounded in $\ell_2$ norm by some $G>0$, and satisfy the curvature condition with some parameter $\alpha > 0$, over the set $3R\ball$, where $R$ denotes the radius of a ball enclosing $\mK$ and centered at the origin. \footnote{The consideration of a set strictly containing $\mK$ (the ball $3R\ball$) in which these assumptions hold is required since our algorithm will query gradients of the loss functions at infeasible points. For ease of presentation we consider the enclosing set $3R\ball$, however this could be very much relaxed to consider a set only slightly larger than $\mK$ in which the assumption needs to hold, see discussion in Section \ref{sec:AssDiscuss}.} 

\end{assumption}



\subsection{Online Newton step with approximately-feasible (matrix) projections}
We now begin to discuss our high-level approach towards efficient LOO-based implementation of the Online Newton Step method. As discussed, our approach builds on the one in \cite{Garber22a}, which considered the Euclidean Online Gradient Descent method, and extends it to ONS which requires non-Euclidean projections according to matrix-induced norms. 

One of our central algorithmic building blocks is an oracle for computing \textit{approximately-feasible projections} onto the feasible set $\mK$ w.r.t. to some matrix-induced norm, which we now define. In the sequel we show how such an oracle could be implemented efficiently using only a LOO for the feasible set $\mK$. 
\begin{definition} \label{def:app_feasible_projection}
Given a convex and compact set $\mK\subset\reals^n$, a positive definite matrix $\A\in\mbS^n_{++}$, and a tolerance $\epsilon > 0$, we say a function $\mathcal{O}_{AFP}(\y,\A,\epsilon,\mK)$ is an \textit{approximately-feasible projection (AFP) oracle} (for the set $\mK$ with parameters $\A,\epsilon$), if for any input point $\y\in\reals^n$, it returns some $\brac{\x,\widetilde{\y}}\in\mK\times\reals^n$ such that i.
for all $\z\in\mK$, $\Vert \widetilde{\y} - \z \Vert_\A \leq \Vert \y - \z \Vert_\A $, and ii.
$\Vert{\x-\widetilde{\y}}\Vert_{\A}^2 \leq \epsilon$.
\end{definition}

Equipped with the concept of an AFP oracle, we can now introduce our second central algorithmic building block --- a template for ONS-style algorithms that only accesses the feasible set $\mK$ through an AFP oracle. As opposed to the standard (projection-based) ONS which maintains a single sequence of feasible points, Algorithm \ref{alg:ONS-WF} maintains two main sequences: one sequence ($\{\widetilde{\y}_m\}_{m\geq 1}\}$) which is infeasible and corresponds to an ONS-style update, and another sequence ($\{\x_m\}_{m\geq 1}\}$)
which is feasible and point-wise close to the previous sequence. We refer to Algorithm \ref{alg:ONS-WF} as a template since it does not explicitly state how to choose the matrices $\A_m, m=1,2,\dots$, used in the algorithm, but only states some restrictions on them. This will be useful later on to derive our two variants: one in which $\A_m$ is based on exact aggregation of gradients (as in standard ONS), and the other which is only a certain approximation via a matrix sketching technique and useful for reducing memory and runtime requirements in case the gradients span (approximately) only a low-dimensional subspace. Finally, Algorithm \ref{alg:ONS-WF} partitions the prediction rounds $1,\dots,T$ into consecutive disjoint blocks of size $K$ (denoted by a subscript of $m$). This will be important to make sure the AFP oracle is called only once every $K$ iterations, which will allow to upper bound the number of LOO calls required to implement it according to our needs.
\begin{algorithm2e}
\KwData{horizon $T$, block length $K$,  learning rate $\eta>0$, initialization parameter $\epsilon_{I}>0 $, error tolerance $\epsilon>0 $, approximately-infeasible projection oracle $\mathcal{O}_{AFP}\brac{\cdot,\cdot,\cdot,\mK}$}
$\x_1=\widetilde{\y}_{1} \gets $ arbitrary point in $\mK$\\
$\A_0 = \epsilon_{I} \I_n$\\
\For{$~ m = 1,\ldots,T/K ~$}{
    Set $\bar{\nabla}_m = \vz$\\ 
    \For{$~ s = 1,\ldots,K ~$}{
    Play $\x^t = \x_{m}$ for $t=(m-1)K+s$  \\
    Set ${\nabla}_t  =\nabla  f_t(\widetilde{\y}_{m})$ and update $\bar{\nabla}_m = \bar{\nabla}_m + {\nabla}_t$ 
    }
    Update $\A_m$ such that $\A_0 \preceq \A_m \preceq \A_{m-1} + \bar{\nabla}_m \bar{\nabla}_m^\top$\\
    Update $\y_{m+1} = \widetilde{\y}_{m} - \eta \A_{m}^{-1} \bar{\nabla}_m$\\
    Set $\brac{\x_{m+1},\widetilde{\y}_{m+1}} \gets \mathcal{O}_{AFP}(\y_{m+1},\A_{m},3\epsilon, \mK)$
}
\caption{Template for Online Newton Step Without Feasibility}\label{alg:ONS-WF}
\end{algorithm2e}
The following lemma states the regret bound of Algorithm \ref{alg:ONS-WF} that will be used to derive all following regret bounds. \begin{lemma}\label{lemma:ONS-WF}
Consider running Algorithm \ref{alg:ONS-WF} with some block size $K \in [T]$ \footnote{without loosing much generality, throughout this paper we assume that the chosen block size $K$ is integer and divides $T$, which will  ease the  analysis. Waiving this convention will only add lower-order terms to our regret bounds}
 and with $\epsilon_I \geq G^2 K^2$, $\eta \geq \max\{ 12KGR, \frac{2K}{\alpha} \}$.
Suppose further that for all $m$ it holds that $\widetilde{\y}_m\in{}3R\ball$. Then,  it holds that
\begin{align*}
    \forall \x \in \mK : ~ \sum_{t=1}^{T} f_t (\x^t) - f_t ( \x)  & \leq \frac{3 \beta   \epsilon}{\epsilon_I} T +  \sqrt{\frac{6\epsilon{}T}{K}\sum_{m=1}^{T/K}  \Vert \bar{\nabla}_m \Vert_{\A_{m}^{-1}}^2 } + \frac{2R^2 \epsilon_I}{\eta} +  \frac{\eta  }{2} \sum_{m=1}^{T/K} \matnorm{{\nabla}_m}{\A_{m}^{-1}}^2 .
\end{align*}
\end{lemma}
The proof which is given in the appendix, at a high level, builds on coupling the standard ONS proof \cite{HazanBook} with the properties of the AFP oracle, to derive a  regret bound on the infeasible sequence $\{\widetilde{\y}_m\}_{\m\geq 1}$. The smoothness assumption on the losses is then used (and only in this proof) to derive a regret bound on the feasible sequence $\{\x^t\}_{t\geq 1}$, without incurring terms which (eventually) will scale worse than $T^{2/3}$.

\section{Efficient LOO-based Approximately-Feasible Projections}\label{sec:AFP}
In this section we turn to discuss the technical heart of the paper --- the efficient construction of an AFP oracle for the feasible set $\mK$ (Definition \ref{def:app_feasible_projection}) using only a linear optimization oracle for $\mK$. As already discussed, we build on the approach of \cite{Garber22a} for Euclidean projection, but expand on it in two ways: i. we extend it to projection w.r.t. matrix-induced norms, as employed by ONS, and ii. we critically improve certain parts of the analysis, which while not being a bottleneck in the analysis of \cite{Garber22a} (which has a $T^{3/4}$ regret bound), are indeed crucial for our faster $T^{2/3}$ regret bounds.

At a high level, the construction relies on the following idea: given an infeasible point $\y$, using only the LOO, we can either construct a generalized hyperplane that separates $\y$ from $\mK$ with sufficient margin (generalized in the sense that it separates w.r.t. to a given positive definite matrix $\A$, see in the sequel), or find a feasible point that is sufficiently close to $\y$ (in terms of the distance induced by the matrix $\A$). In case such a generalized hyperplane is found, it can then be used to ``pull'' the infeasible point closer to $\mK$, and the process repeats itself. 

We show that by applying the classical LOO-based Frank-Wolfe method \cite{FrankWolfe, Jaggi13} to the non-Euclidean projection problem $\min_{\x\in\mK}\Vert{\x-\y}\Vert_{\A}^2$, we can indeed either find such a separating hyperplane, or find a close-enough feasible point, w.r.t. the matrix $\A$. 

One may wonder: \textit{if we can directly approximate matrix-based projections, arbitrarily well, using Frank-Wolfe, why do we need to go through the (conceptually more complex) approach of using separating hyperplanes?} The reason is that, has already discussed in \cite{Garber22a}, such a simplified approach will lead to a worse regret/oracle complexity tradeoff (mainly in terms of $T$). In particular, when applying Frank-Wolfe to the problem $\min_{\x\in\mK}\Vert{\x-\y}\Vert_{\A}^2$, we will only compute a feasible point that is an approximated projection. On the other hand, with our approach (recall the definition of the AFP oracle) we always return a valid (though infeasible) projection (and a feasible point that is sufficiently close to it), which allows for a tighter regret analysis.

The following lemma shows how given an infeasible point $\y$ and such a generalized separating hyperplane, we can ``pull'' $\y$ closer to the feasible set.
\begin{lemma}\label{lemma:update_step_with_hp}
   Let $\mK\subset\reals^n$ be convex and compact, let $\A \in \mbS^n_{++}$,  and let $\y\in\reals^n\setminus\mK$. Let $\g\in\reals^n$ be such that for all $\z\in\mK$, $(\y-\z)^{\top} \A \g \geq Q$, for some $Q \geq 0$. Consider the point $\widetilde{\y} = \y - \gamma \g$ for $\gamma = Q/C^2$, where $C \geq \Vert{\g}\Vert_{\A}$. It holds that
\begin{align*}
   \forall \z\in\mK: \quad \Vert \widetilde{\y} -\z \Vert_{\A}^2 \leq \left\Vert \y -\z  \right\Vert_{\A}^2 - (Q/C)^2.
\end{align*}
\end{lemma}

\begin{proof}
Fix some $\z\in\mK$. It holds that
\begin{align*}
    \Vert \widetilde{\y} -\z \Vert_{\A}^2 = \left\Vert \y -\z - \gamma  \g \right\Vert_{\A}^2 = \left\Vert \y -\z  \right\Vert_{\A}^2 - 2 \gamma (\y -\z )^\top \A \g + \gamma^2 \left\Vert \g \right\Vert_{\A}^2.
\end{align*}
Since $\left( \y - \z \right)^\top \A \g \geq Q$ and $C \geq \left\Vert \g \right\Vert_{\A}$, we indeed obtain 
\begin{align*}
    \Vert \widetilde{\y} -\z \Vert_{\A}^2 \leq \left\Vert \y -\z  \right\Vert_{\A}^2 - 2 \gamma Q + \gamma^2 C^2 = \Vert \y -\z \Vert_{\A}^2 - Q^2/C^2,
\end{align*}
where the last equality follows from plugging-in the value of $\gamma$.
\end{proof}

Algorithm \ref{alg:SH-FW} given below, which simply applies the Frank-Wolfe method (with line-search) for smooth convex minimization over a convex and compact set  \cite{Jaggi13} to the non-Euclidean projection problem  $\min_{\x\in\mK}\Vert{\x-\y}\Vert_{\A}^2$, returns some feasible point $\widetilde{\x}\in\mK,$ that is either close enough (w.r.t. $\Vert{\cdot}\Vert_{\A})$ to the infeasible point $\y$, or can be used to construct a hyperplane which separates $\y$ from $\mK$ w.r.t. $\A$ and with sufficient margin. 

\begin{algorithm2e}
  \KwData{LOO for the feasible set $\mK$, error tolerance $\epsilon>0$, initial point $\x_1 \in \mK$, $\A\in\mbS^n_{++}$,  infeasible point $\y$}
  \For{ $i =1,2, \dots$}{
        $ \mathbf{v}_{i} \in \argmin\limits_{\x \in \mK} \{ (\x_{i} - \y)^{\top} \A \x \} $\tcc*{call to LOO of $\mK$}
        \uIf{$( \x_i - \y )^\top \A (\x_i -\v_i) \leq \epsilon$ or $\Vert \x_{i} - \y \Vert_{\A}^2 \leq 3\epsilon$}{
	        \textbf{return} $\widetilde{\x} \gets \x_{i}$
        }
	    $ \sigma_{i} = \argmin\limits_{\sigma \in [0, 1]}  \{ \Vert \y - \x_{i} - \sigma (\mathbf{v}_i - \x_{i})) \Vert_{\A}^2 \}$\\
	$ \x_{i+1} = \x_i + \sigma_{i} (\mathbf{v}_i - \x_i) $\\
    }
  \caption{Generalized Separating Hyperplane via Frank-Wolfe}\label{alg:SH-FW}
\end{algorithm2e}

\begin{lemma} \label{lemma:SH-FW} 
Algorithm \ref{alg:SH-FW} terminates after at most $\left\lceil \brac{27 R^2 \lambda_1 (\A) / \epsilon } -2 \right\rceil$ iterations, and returns a point $\widetilde{\x} \in \mK$  satisfying:
\begin{enumerate}
\item
$\Vert \widetilde{\x} - \y \Vert_{\A}^2 \leq \Vert \x_1 - \y \Vert_{\A} ^2$.
\item At least one of the following holds: $\Vert \widetilde{\x} - \y \Vert_{\A}^2 \leq 3\epsilon$ or  $\forall \z \in \mK:  (\y - \z)^\top \A (\y - \widetilde{\x}) > (2/3) \Vert \widetilde{\x} - \y \Vert_{\A}^2$.
\item If $\dist_\A^2 (\y, \mK) \leq \epsilon$, then $\Vert \widetilde{\x} - \y \Vert_{\A}^2 \leq 3\epsilon$.
\end{enumerate}

\end{lemma}
\begin{proof}
As discussed, Algorithm \ref{alg:SH-FW} is simply the well-known Frank-Wolfe method with line-search, see Algorithm 3 in \cite{Jaggi13}, when applied to minimizing the convex and $\lambda_1(\A)$-smooth function $g(\x) := \frac{1}{2}\Vert{\x-\y}\Vert_{\A}^2$, whose gradient vector is given by $\nabla g(\x) = \A(\x-\y)$, over the set $\mK$. 
Thus, the upper-bound on the number of iterations executed by Algorithm \ref{alg:SH-FW} follows immediately from  Theorem 2 in \cite{Jaggi13}, which gives a convergence rate for the dual gap. For our choice of $g$, the dual gap on any iteration $i$ is given precisely by $\nabla{}g(\x_i)^{\top}(\x_i-\v_i) = ( \x_i - \y )^\top \A (\x_i -\v_i)$, which corresponds to one of the stopping conditions is Algorithm \ref{alg:SH-FW}.

Since the line-search guarantees that the function value $g(\x_i) = \frac{1}{2}\Vert{\x_i-\y}\Vert_{\A}^2$ does not increase when moving from iterate $\x_i$ to $\x_{i+1}$,  Item 1 holds trivially.

Item 2 follows from the stopping condition of the algorithm and by noting that, if for some iteration $i$ it  holds that $(\x_i - \y )^\top \A (\x_i -\v_i) \leq \epsilon$ and $\Vert{\x_i-\y}\Vert_{\A}^2 > 3\epsilon$ (in which case the algorithm will return $\widetilde{\x} = \x_i$) then, for all $ \z\in\mK$ it holds that
\begin{align*}
  \left( \z - \y \right)^\top \A \left( \x_{i} - \y \right) &  = \left( \z - \x_{i} \right)^\top \A \left( \x_{i} - \y \right) + \Vert \x_{i} - \y \Vert_{\A}^2 \geq \left( \v_{i} - \x_{i} \right)^\top \A \left( \x_{i} - \y \right) + \Vert \x_{i} - \y \Vert_{\A}^2 \\ 
  & \geq -\epsilon  + \Vert \x_{i} - \y \Vert_{\A}^2 > -(\Vert \x_{i} - \y \Vert_{\A}^2/3) + \Vert \x_{i} - \y \Vert_{\A}^2  = (2/3) \Vert \x_{i} - \y \Vert_{\A}^2 ,
\end{align*}
where the first inequality is due to the definition of $\v_i$. 

Finally, to prove Item 3, denote $\x^* = \argmin_{\x\in\mK}\Vert{\x-\y}\Vert_{\A}^2$. Suppose by contradiction that $\dist_{\A}^2(\y,\mK) = \Vert{\x^*-\y}\Vert_{\A}^2 \leq \epsilon$, and  $\Vert{\widetilde{\x}-\y}\Vert_{\A}^2 > 3\epsilon$. By the stopping condition of the algorithm, on the last iteration executed $i$, it must hold that $(\widetilde{\x}-\y)^{\top}\A(\widetilde{\x}-\v_i) = \max_{\v\in\mK}\nabla{}g(\widetilde{\x})^{\top}(\widetilde{\x}-\v) \leq \epsilon$, which means that
\begin{align*}
	\Vert{\widetilde{\x}-\y}\Vert_{\A}^2 - \dist_{\A}^2(\y,\mK) = 2g(\widetilde{\x}) - 2g(\x^*) \leq 2\nabla{}g(\widetilde{\x})^{\top} (\widetilde{\x} - \x^*) \leq  2 \max_{\v\in\mK} \nabla{}g(\widetilde{\x})^{\top} (\widetilde{\x}-\v) \leq 2\epsilon,
\end{align*}
where the first inequality is due to the gradient inequality and the convexity of $g(\cdot)$. Thus, we have that $\Vert{\widetilde{\x}-\y}\Vert_{\A}^2 \leq 2\epsilon + \dist_{\A}^2(\y,\mK) \leq 3\epsilon$, which contradicts the assumption that $\Vert{\widetilde{\x}-\y}\Vert_{\A}^2 > 3\epsilon$.
\end{proof}

Our LOO-based implementation of a AFP oracle for the feasible set $\mK$ is given as Algorithm \ref{alg:CIP-FW}. The algorithm builds on iteratively using separating hyperplanes generated by Algorithm \ref{alg:SH-FW} to ``pull closer'' the infeasible point $\y$ towards the feasible set $\mK$ using the updates suggested in Lemma \ref{lemma:update_step_with_hp}, until it is sufficiently close.

 \begin{algorithm2e}
\KwData{LOO for the feasible set $\mK$, feasible point $\x_{0} \in \mK$, initial point $\y_{1}\in\reals^n$, $\A\in\mbS^n_{++}$, error tolerance $\epsilon>0$, step-size $\gamma > 0$}
  \If{$\Vert \x_{0} - \y_{1} \Vert_{\A}^2 \leq 3\epsilon$}{
        \textbf{Return} $\x \gets \x_{0}$, $\y \gets \y_{1}$
    }
    \For{$i=1,2, \dots$}{
    $\x_{i} \gets$ Output of Algorithm \ref{alg:SH-FW} when called with LLO of $\mK$, tolerance $\epsilon$, feasible point $\x_{i-1}$, positive definite matrix $\A$, and  initial point $\y_{i}$\\
    \eIf{$\Vert \x_{i} - \y_{i} \Vert_{\A}^2 > 3\epsilon$}{
        $\y_{i+1} =  \y_{i} - \gamma \left( \y_{i} - \x_{i} \right)$
    }
    {
    \textbf{Return} $\x \gets  \x_{i}$, $\y \gets \y_i$
    }
  }
  \caption{Approximately-Feasible (matrix) Projection via a Linear Optimization Oracle}\label{alg:CIP-FW}
\end{algorithm2e}

The proof of the following lemma is given in the appendix.
\begin{lemma} \label{lemma:CIP-FW}
Setting $\gamma= 2/3$ in Algorithm \ref{alg:CIP-FW} guarantees that it stops after at most 
\begin{align*}
    \max \left\{2.25\log\brac{ \frac{\Vert \y_{1} -\x_{0} \Vert_{\A}^2 }{ \epsilon} }+1, 0 \right\}
\end{align*}
 iterations, and returns $(\x,\y) \in \mK\times \brac{R +  \sqrt{3 \epsilon /\lambda_{n}(\A)} }\ball$ such that 
\begin{align*}
    \forall \z \in \mK : ~ \Vert \y - \z \Vert_{\A}^2 \leq  \Vert \y_{1} - \z \Vert_{\A}^2 ~~~~ \text{and} ~~~~~   \Vert \x - \y \Vert_{\A}^2 \leq 3\epsilon.
\end{align*}
\end{lemma}
It is important to note that Lemma \ref{lemma:CIP-FW} significantly and critically improves upon its Euclidean counterpart in \cite{Garber22a}: while the number of iterations here scales only with $\log(1/\epsilon)$, in \cite{Garber22a} it scales with $1/\epsilon^2$. This improvement is critical for obtaining our improved regret/oracle complexity tradeoffs.

\section{LOO-based Online Newton Step}\label{sec:ONS}
In this section we present our main result --- an efficient LOO-based ONS-style algorithm and its regret and complexity guarantees.

The following lemma builds on the combination of our ONS Without Feasibility template (Algorithm \ref{alg:ONS-WF}) together with our LOO-based construction for an AFP oracle (Algorithm \ref{alg:CIP-FW}). The proof is given in the appendix.
\begin{lemma}\label{lem:LOO-ONS}
Fix block size $K\in[T]$. Consider running Algorithm \ref{alg:ONS-WF}  with parameters $\eta, \epsilon, \epsilon_I$  such that $\eta \geq \max\{ 12KGR, \frac{2K}{\alpha} \}, \epsilon_I \geq (KG)^2$, and $\frac{3\epsilon}{\epsilon_I} \leq 4R^2$, and when the $\mathcal{O}_{AFP}$ oracle is implemented via Algorithm \ref{alg:CIP-FW}, where  the initial feasible input to Algorithm \ref{alg:CIP-FW} (the point $\x_0$ in Algorithm \ref{alg:CIP-FW}), when called during block $m$ in Algorithm \ref{alg:ONS-WF}, is the previous feasible output of  Algorithm \ref{alg:CIP-FW} --- the point $\x_m$, if $m \geq 2$, and the initialization point of Algorithm \ref{alg:ONS-WF}  (the point $\x_1$), if $m=1$. Then, the regret  is upper bounded by
\begin{align*}
    \sum_{t=1}^{T} f_t(\x^t) - \min_{\x^* \in \mK} \sum_{t=1}^{T} f_t(\x^*) \leq  \frac{3 \beta   \epsilon}{\epsilon_I} T + \sqrt{\frac{6\epsilon{}T}{K}\sum_{m=1}^{T/K}  \Vert \bar{\nabla}_m \Vert_{\A_{m}^{-1}}^2 } + \frac{2R^2 \epsilon_I}{\eta} +  \frac{\eta  }{2} \sum_{m=1}^{T/K} \matnorm{{\nabla}_m}{\A_{m}^{-1}}^2 ,
\end{align*}
and  the overall number of calls to the LOO of $\mK$ is upper bounded by 
\begin{align*}
    N_{calls} & \leq 61 R^2 \log \left( 19  + 4 \frac{ \eta^2 K^2 G^2}{\epsilon  \epsilon_I} \right) \frac{\epsilon_I + G^2KT}{K \epsilon}T.
\end{align*}
\end{lemma}

We are now ready to formally present our main result. Here for ease of presentation we present a concise version only. A fully detailed version which specifics all choices of parameters and all poly-logarithmic factors, as well as the proof, is given in the appendix.  
\begin{theorem}[short version]\label{thm:mainthm:short} 
Consider the implementation of Algorithm \ref{alg:ONS-WF} as described in Lemma \ref{lem:LOO-ONS} and when using the (standard ONS) update rule: $\A_m = \A_{m-1} + \bar{\nabla}_m \bar{\nabla}_m^\top$ for every block $m$.
\begin{enumerate}
\item
If $T\geq T_0 = \widetilde{O}(1)$, there exists a choice for the parameters $K,\eta,\epsilon,\epsilon_I$ in Algorithm \ref{alg:ONS-WF} which depends only on the quantities $T,n,G,R,\alpha$ and satisfies the assumptions of Lemma \ref{lem:LOO-ONS}, such that the regret is upper-bounded by
\begin{align}\label{eq:mainres:1}
\R_T =  \widetilde{O}\left({(\beta{}R^2 + (GR+\alpha^{-1})n^{2/3})T^{2/3}}\right). 
\end{align}
\item
In continue to the previous item and under the same choice of parameters, for any $\rho\in[n]$, denoting $\Omega_{\rho} = \sum_{i=\rho+1}^n\lambda_i(\sum_{t=1}^T\nabla_t\nabla_t^{\top})$ ($\nabla_t$ is as defined in Algorithm \ref{alg:ONS-WF}), the regret is upper-bounded by
\begin{align}\label{eq:mainres:2} 
\R_T 
&= \widetilde{O}\left({(\beta{}R^2 + GR(\rho^{1/2}n^{1/6}+n^{1/3})+\alpha^{-1}n^{-1/3}\rho)T^{2/3}}\right) \nonumber  \\
&~ + \widetilde{O}\left({RT^{1/3}\sqrt{\Omega_{\rho}} + G^{-2}n^{-2/3}(GR+\alpha^{-1})\Omega_{\rho}}\right). 
\end{align}
\item
Fix $\rho\in[n]$. If $T\geq T_0 = \widetilde{O}(1)$, there exists a choice for the parameters $K,\eta,\epsilon,\epsilon_I$ in Algorithm \ref{alg:ONS-WF} which depends only on the quantities $T,n,G,R,\alpha$ and $\rho$, and satisfies the assumptions of Lemma \ref{lem:LOO-ONS}, such that the regret is upper-bounded by
\begin{align}\label{eq:mainres:3} 
\R_T = \widetilde{O}\left({(\beta{}R^2 + (GR+\alpha^{-1})\rho^{2/3})T^{2/3}+RT^{1/3}\sqrt{\Omega_{\rho}} + G^{-2}\rho^{-2/3}(GR+\alpha^{-1})\Omega_{\rho}}\right).
\end{align}
Note this bound is not explicitly dependent on the ambient dimension $n$.
\end{enumerate}
In all cases, the overall number of calls to the LOO of $\mK$ is upper-bounded by $O(T+n^{1/3}T^{2/3})$, the additional space requirement in $O(n^2)$, and using the Sherman-Morrison formula for fast matrix inversion, the overall additional runtime is $O(n^2(T+n^{1/3}T^{2/3}))$.
\end{theorem}
Let us make a few comments  regarding Theorem \ref{thm:mainthm:short}. The regret bounds \eqref{eq:mainres:2}, \eqref{eq:mainres:3} may significantly improve upon the worst case bound \eqref{eq:mainres:1} in case the observed gradients approximately span a subspace of dimension at most $\rho$, for some $\rho\in[n]$, in the sense that $\Omega_{\rho} = O(T^{2/3})$ (note that $\Omega_{\rho}=0$ implies that the dimension of the subspace spanned by the gradients is at most $\rho$). In particular, the bound  \eqref{eq:mainres:2} holds simultaneously for all values of $\rho$ (i.e., the algorithm is independent of the choice of $\rho$), but still depends on the ambient dimension $n$ (though with milder dependence than \eqref{eq:mainres:1}), while the bound  \eqref{eq:mainres:3} is completely independent of $n$, but requires a priori knowledge of $\rho$. In case it indeed holds that $\Omega_{\rho} = O(T^{2/3})$ for some known $\rho << n$, \eqref{eq:mainres:3} translates into a $ \widetilde{O}\left({(\beta{}R^2 + (GR+\alpha^{-1})\rho^{2/3})T^{2/3}}\right)$ regret bound.


\section{Leveraging Frequent Directions Sketching for Low-dimensional Data}\label{sec:sketch}
While Theorem \ref{thm:mainthm:short} yields a regret bound for Algorithm \ref{alg:ONS-WF} which is independent of the ambient dimension $n$ and depends only on the (approximate) dimension of the subspace spanned by the gradients (guarantee \eqref{eq:mainres:3}), the space and average additional runtime requirements still scale with $n^2$. Following the approach of  \cite{luo2016efficient}, who considered the coupling of ONS with matrix sketching techniques to reduce space and runtime requirements in case of low-dimensional data (but not in a projection-free setting),  in this section we discuss the implications of such coupling to our LOO-based algorithm.

Similarly to \cite{luo2016efficient}, we consider the use of the well known deterministic \textit{Frequent Directions} sketching method \cite{ghashami2016frequent}. The idea is that instead of taking the matrix $\A_m$ for each block $m$ in Algorithm \ref{alg:ONS-WF} to be the exact aggregation of gradients as in Theorem \ref{thm:mainthm:short} and maintain it (and its inverse $\A_m^{-1}$) explicitly,  we shall  only maintain a certain approximation of this gradient information in a low-rank factorized form, see Algorithm \ref{alg:FD-S-ONS} which shows how the Frequent Directions sketch is used in synergy with Algorithm \ref{alg:ONS-WF}.


\begin{algorithm2e}[!ht]
\KwData{sketch size  $\rho \in [n]$, $\epsilon_I >0$}
\textbf{Initialization: }
Set $\S_0 = \vz_{(\rho+1) \times n}$, and $\A_0 = \epsilon_I \I_n$\\
\For{$m=1$ to $T/K$}{
    Receive $\bar{\nabla}_m \in \reals^n$ from Algorithm \ref{alg:ONS-WF} and insert it as the last row of $\S_{m-1}$ \\
    Compute eigendecomposition of $\S_{m-1}^\top \S_{m-1}$: $\V_{m}^\top \widehat{\Sigma}_m \V_{m} = \S_{m-1}^\top \S_{m-1}$\\
    Set $\sigma_m = \widehat{\Sigma}_m\brac{\rho+1,\rho+1}$ and $\Sigma_{m} = \widehat{\Sigma}_m - \sigma_m \I_{\rho+1}$ \comment{$\Sigma_m(\rho+1,\rho+1) = 0$}\\
    Set $\S_m = \brac{\Sigma_m}^\frac{1}{2} \V_m$   \comment{ last row of $\S_{m}$ is now $\vz$}\\
    Set $\H_m = \diag{\brac{\frac{1}{\epsilon_I + \Sigma_m(1,1) }, \dots, \frac{1}{\epsilon_I + \Sigma_m(\rho,\rho) }, \frac{1}{\epsilon_I} }}$\comment{$\H_m = \brac{\epsilon_I \I_{\rho+1} + \S_m \S_m^\top}^{-1}$}\\
    Set $\A_m = \A_0 + \S_m^\top \S_m$, $\A_m^{-1} = \epsilon_I^{-1} \brac{\I_n - \S_m^\top \H_m \S_m } $ \comment{not to be explicitly computed; the expression for $\A_m^{-1}$ follows from the Woodbury matrix identity}
}
\caption{Frequent Directions Sketch for Algorithm \ref{alg:ONS-WF}}\label{alg:FD-S-ONS}
\end{algorithm2e}

The full version of the following theorem, as well as the proof and additional details regarding Algorithm \ref{alg:FD-S-ONS}, are given in the appendix.
\begin{theorem}\label{thm:LOO-ONS-FDS}
Fix $\rho\in[n]$. Consider the implementation of Algorithm \ref{alg:ONS-WF} as described in Lemma \ref{lem:LOO-ONS}, and when the matrix $\A_m$ for every block $m$ in Algorithm \ref{alg:ONS-WF}  is generated by Algorithm \ref{alg:FD-S-ONS}. Denote $\Omega_\rho = \sum_{i=\rho+1}^{n} \lambda_i \brac{\sum_{t=1}^{T} {\nabla}_t {\nabla}_t^\top}$. If $T \geq T_0 = \widetilde{O}(1)$, then there exists a choice for the parameters $K,\eta,\epsilon,\epsilon_I$ in Algorithm \ref{alg:ONS-WF} which depends only on the quantities $T,\rho,G,R,\alpha$ and satisfies the assumptions of Lemma \ref{lem:LOO-ONS}, such that the regret is upper bounded by 
\begin{align*}
\R_T = \widetilde{O}\brac{ \brac{\beta R^2+\brac{ GR+ \alpha^{-1}}\rho^{2/3} }T^{2/3} + R \rho^{1/2} T^{1/3}   \sqrt{ \Omega_\rho } + G^{-2} \rho^{1/3}  \brac{ GR + \alpha^{-1} }   \Omega_\rho }.
\end{align*}
The overall number of calls to the LOO is upper bounded by $O\brac{  \rho^{1/3} T^{2/3}  + T}$, the additional space requirement in $O(\rho n)$, and the overall additional runtime is $O(\rho n T + \rho^{4/3}nT^{2/3} + \rho^{7/3}nT^{1/3} )$.
\end{theorem}

\section{Discussion}
We provided the first projection-free LOO-based algorithm for exp-concave and smooth losses that in the case of (approximately) low-dimensional gradients, using  $O(T)$ queries to the LOO, guarantees regret that  both scales only with $T^{2/3}$, and is
independent of the ambient dimension.

It is interesting if a similar result could be obtained when removing one or more of the above assumptions: smoothness of the losses, exp-concavity of the losses, low-dimensionality of the gradients. In particular, the two recent works \cite{hazan2020faster, mhammedi2022} achieve fast LLO-based regret bounds that scale with $T^{2/3}$ (but also with the dimension) without curvature assumptions on the losses. It is thus interesting whether the exp-concavity assumption, or even strong convexity \cite{kretzu2021revisiting}, could lead to even faster rates than $T^{2/3}$.


\bibliography{bib}

\appendix

\section{Proof of  Lemma \ref{lemma:ONS-WF}}
\begin{proof}
Fix some block $m\in\{1,\dots,T/K\}$. Since $\bar{\nabla}_m = \sum_{t=(m-1)K+1}^{mK} {\nabla}_t$, it holds that $\bar{\nabla}_m \bar{\nabla}_m^\top \preceq K^2 G^2 \I_n$. Thus, for the value of $\epsilon_I$ stated in the lemma and the choice of $\A_{m-1}$, we have that $\bar{\nabla}_m \bar{\nabla}_m^\top \preceq \epsilon_I \I_n = \A_0 \preceq \A_{m-1}$. Additionally, since $\A_m \preceq \A_{m-1} + \bar{\nabla}_m \bar{\nabla}_m^\top$, we have that
\begin{align}
    \Vert \x_{m} - \widetilde{\y}_{m} \Vert_{\A_{m}}^2 \leq \Vert \x_{m} - \widetilde{\y}_{m} \Vert_{\A_{m-1}}^2 + \Vert \x_{m} - \widetilde{\y}_{m} \Vert_{\bar{\nabla}_m \bar{\nabla}_m^\top}^2 \leq  2 \Vert \x_{m} - \widetilde{\y}_{m} \Vert_{\A_{m-1}}^2. \label{eq:Am_to_Am_minus_one1}
\end{align}
Denote $\x^* \in \argmin_{\x \in \mK} \sum_{t=1}^{T} f_t(\x)$, $m(t) := \left\lceil \frac{t}{K} \right\rceil$, and $\g_t = \nabla f_t(\x_{m(t)})$ for all $t\in[T]$. Using the convexity of each $f_t(\cdot)$, it holds that
\begin{align*}
    \sum_{t=1}^{T} f_t(\x_{m(t)}) - f_t(\x^*) 
    & \leq \sum_{t=1}^{T} \g_t^\top  \left( \x_{m(t)} - \widetilde{\y}_{m(t)} \right) + \sum_{t=1}^{T} f_t(\widetilde{\y}_{m(t)}) - f_t(\x^*). 
\end{align*}
Since each $f_t(\cdot)$ is $\beta$-smooth we have that,
\begin{align}
    \sum_{t=1}^{T} f_t(\x_{m(t)}) - f_t(\x^*) & \leq \sum_{t=1}^{T} (\g_t - {\nabla}_t)^\top \left( \x_{m(t)} - \widetilde{\y}_{m(t)} \right) + {\nabla}_t^{\top} \left( \x_{m(t)} - \widetilde{\y}_{m(t)} \right)  + f_t(\widetilde{\y}_{m(t)}) - f_t(\x^*) \nonumber \\
     & \leq \sum_{t=1}^{T} \beta \enorm{ \x_{m(t)} - \widetilde{\y}_{m(t)}}^2 + \sum_{m=1}^{T/K} \bar{\nabla}_m \left( \x_{m} - \widetilde{\y}_{m} \right)  + \sum_{t=1}^{T} f_t(\widetilde{\y}_{m(t)}) - f_t( \x^* ). \label{eq:full_ons_regret_analysis_loo}
\end{align}
Since for every block $m$, $(\x_m,\widetilde{\y}_m)$ are the output of $\mathcal{O}_{AFP}(\y_m,\A_{m-1},3\epsilon,\mK)$, we have that

\begin{align}
   \sum_{t=1}^{T} \beta \enorm{ \x_{m(t)} - \widetilde{\y}_{m(t)}}^2 = K \sum_{m=1}^{T/K} \beta \enorm{ \x_{m} - \widetilde{\y}_{m}}^2  \underset{(a)}{\leq} \frac{\beta  K}{\epsilon_I} \sum_{m=1}^{T/K}\matnorm{ \x_{m} - \widetilde{\y}_{m}}{\A_{m-1}}^2  \underset{(b)}{\leq} \frac{3 \beta   \epsilon}{\epsilon_I} T,
    \label{eq:smooth_part}
\end{align}
    where (a) holds since, by the choice of  $\A_{m-1}$, it holds that $\A_{m-1}\succeq \A_0 = \epsilon_I\I_n$, and so, for any $\z\in\reals^n$ we have that $\Vert{\z}\Vert_{\A_{m-1}}^2 \geq \epsilon_I\Vert{\z}\Vert^2$, and (b) holds due to the gurantee of the AFP oracle.

Using Eq. \eqref{eq:Am_to_Am_minus_one1} and the facts that for every block $m$, $\Vert{\cdot}\Vert_{\A_m}, \Vert{\cdot}\Vert_{\A_m^{-1}}$ are dual norms, and $(\x_m, \widetilde{\y}_m)$ are the outputs of $\mathcal{O}_{AFP}$, we have that
\begin{align*}
     \sum_{m=1}^{T/K} \bar{\nabla}_m^\top  \left( \x_{m} - \widetilde{\y}_{m} \right)  &\leq \sum_{m=1}^{T/K} \Vert \bar{\nabla}_m \Vert_{\A_{m}^{-1}} ~ \Vert \x_{m} - \widetilde{\y}_{m} \Vert_{\A_{m}} \\ 
     &\leq \sum_{m=1}^{T/K} \Vert \bar{\nabla}_m \Vert_{\A_{m}^{-1}} ~ \sqrt{2}\Vert \x_{m} - \widetilde{\y}_{m} \Vert_{\A_{m-1}} 
      \leq  \sqrt{6\epsilon} \sum_{m=1}^{T/K} \Vert \bar{\nabla}_m \Vert_{\A_{m}^{-1}},
\end{align*}
where the last inequality is again due to the guarantee of the AFP oracle.

Using Jensen's inequality, we have that
\begin{align}
     \sum_{m=1}^{T/K} \bar{\nabla}_m^\top  \left( \x_{m} - \widetilde{\y}_{m} \right)  \leq  \sqrt{6\epsilon} \sqrt{\left(\sum_{m=1}^{T/K}  \Vert \bar{\nabla}_m \Vert_{\A_{m}^{-1}}\right)^2} \leq \sqrt{6\epsilon}  \sqrt{\frac{T}{K}} \sqrt{\sum_{m=1}^{T/K}  \Vert \bar{\nabla}_m \Vert_{\A_{m}^{-1}}^2 }.
    \label{eq:LOO-ONS_x_y_regret}
\end{align}
Now we turn to upper bound the third term in Eq. \eqref{eq:full_ons_regret_analysis_loo}. For every block $m$, using the fact that $\widetilde{\y}_{m+1}$ is the output of $\mathcal{O}_{AFP}$ w.r.t. the input point $\y_{m+1}$, we have that
\begin{align*}
    \forall\x\in\mK: \quad  \Vert \widetilde{\y}_{m+1}  - \x \Vert_{\A_{m}}^2  & \leq \Vert  \y_{m+1} - \x \Vert_{\A_{m}}^2 = \Vert    \widetilde{\y}_{m} -  \eta  \A_{m}^{-1} \bar{\nabla}_m - \x \Vert_{\A_{m}}^2 \\
    & = \left\Vert \widetilde{\y}_{m} - \x \right\Vert_{\A_{m}}^2 + \eta^2 \bar{\nabla}_m^\top \A_{m}^{-1} \bar{\nabla}_m - 2 \eta \bar{\nabla}_m^\top (\widetilde{\y}_{m} - \x).
\end{align*}
Rearranging, for every $m \in [T/K]$, we have that
\begin{align}
    \bar{\nabla}_m^\top (\widetilde{\y}_{m} - \x) \leq  \frac{ \left\Vert   \widetilde{\y}_{m} - \x \right\Vert_{\A_{m}}^2}{2\eta} - \frac{ \left\Vert   \widetilde{\y}_{m+1} - \x \right\Vert_{\A_{m}}^2}{2\eta} + \frac{\eta  }{2} \Vert \bar{\nabla}_m \Vert_{\A_{m}^{-1}}^2.\label{eq:ONS-WF_one_iteration_bound}
\end{align}
Before we continue, we upper bound the summation of $ \left\Vert   \widetilde{\y}_{m} - \x \right\Vert_{\A_{m}}^2 -  \left\Vert   \widetilde{\y}_{m+1} - \x \right\Vert_{\A_{m}}^2$ over $m \in [T/K]$. Since $ \widetilde{\y}_1\in \mK $ and $\A_0 =\epsilon_I\I_n$, we have that
\begin{align*}
    \sum_{m=1}^{T/K}  \left\Vert   \widetilde{\y}_{m} - \x \right\Vert_{\A_{m}}^2 - \left\Vert   \widetilde{\y}_{m+1} - \x \right\Vert_{\A_{m}}^2 & \leq \left\Vert   \widetilde{\y}_{1} - \x \right\Vert_{\A_{1}}^2 + \sum_{m=2}^{T/K} (\widetilde{\y}_{m}  - \x)^\top \left(\A_{m}-\A_{m-1}\right) (\widetilde{\y}_{m}  - \x)\\
    & = \left\Vert   \widetilde{\y}_{1} - \x \right\Vert_{\A_{0}}^2 + \sum_{m=1}^{T/K} (\widetilde{\y}_{m}  - \x)^\top \brac{\A_{m}-\A_{m-1}}(\widetilde{\y}_{m}  - \x) \\
    & \leq 4R^2 \epsilon_I + \sum_{m=1}^{T/K} (\widetilde{\y}_{m}  - \x)^\top \brac{\A_{m}-\A_{m-1}} (\widetilde{\y}_{m}  - \x).
\end{align*}
Summing Eq. \eqref{eq:ONS-WF_one_iteration_bound} over $m\in[T/K]$, and using the fact that for all blocks $m$, $\A_m - \A_{m-1} \preceq \bar{\nabla}_m \bar{\nabla}_m^\top $, we have that
\begin{align*}
    \sum_{m=1}^{T/K} \bar{\nabla}_{m}^\top (\widetilde{\y}_{m} - \x) \leq \frac{2R^2\epsilon_I}{\eta} + \sum_{m=1}^{T/K} \frac{ (\widetilde{\y}_{m}  - \x)^\top \bar{\nabla}_m \bar{\nabla}_m^\top (\widetilde{\y}_{m}  - \x)}{2\eta} +   \frac{\eta  }{2} \sum_{m=1}^{T/K} \Vert \bar{\nabla}_m \Vert_{\A_{m}^{-1}}^2.
\end{align*}
Since $\eta \geq \max\{ 12KGR, \frac{2K}{\alpha} \}$, and $\widetilde{\y}_m \in 3R\ball$ for every $m \in [T/K]$ (by the assumption of the lemma), using Lemma \ref{lemma:block_property} (sum of functions which satisfy the curvature condition in Definition \ref{def:exp_concave_property}, also satisfies this condition) we have that for every $\x \in \mK \subseteq 3R\ball$, and every $m\in[T/K]$, it holds that
\begin{align*}
    \sum_{t=(m-1)K+1}^{mK} f_t(\widetilde{\y}_{m}) - f_t(\x) \leq \bar{\nabla}_m^\top \left( \widetilde{\y}_{m} - \x \right) - \frac{1}{2\eta} \left( \widetilde{\y}_{m} - \x \right)^\top \bar{\nabla}_m \bar{\nabla}_m^\top \left( \widetilde{\y}_{m} - \x \right).
\end{align*}
Combining the last two inequalities we obtain that for every $\x\in\mK$ it holds that,
\begin{align*} 
    \sum_{t=1}^{T} f_t(\widetilde{\y}_{m(t)}) -  f_t(\x) \leq & \frac{2R^2\epsilon_I}{\eta} +  \frac{\eta  }{2} \sum_{m=1}^{T/K} \matnorm{{\nabla}_m}{\A_{m}^{-1}}^2 .
\end{align*}
Plugging-in the last equation, Eq. \eqref{eq:LOO-ONS_x_y_regret}, and Eq. \eqref{eq:smooth_part} into Eq. \eqref{eq:full_ons_regret_analysis_loo}, we obtain
\begin{align*}
    \sum_{t=1}^{T} f_t(\x_{m(t)}) - f_t(\x^*) &  \leq \frac{3 \beta   \epsilon}{\epsilon_I} T + \sqrt{6\epsilon}  \sqrt{\frac{T}{K}} \sqrt{\sum_{m=1}^{T/K}  \Vert \bar{\nabla}_m \Vert_{\A_{m}^{-1}}^2 } + \frac{2R^2 \epsilon_I}{\eta} +  \frac{\eta  }{2} \sum_{m=1}^{T/K} \matnorm{{\nabla}_m}{\A_{m}^{-1}}^2.
\end{align*}
Finally, the lemma follows from recalling that with the notation of Algorithm \ref{alg:ONS-WF}, we have that for all $t\in[T]$, $\x^t = \x_{m(t)}$.
\end{proof}

\section{Proof of  Lemma \ref{lemma:CIP-FW}}
\begin{proof}
First, we note that in the special case that $\Vert \x_{0} - \y_{1} \Vert_{\A}^2 \leq 3\epsilon$, since $ \Vert \y_{1} \Vert - \Vert \x_{0} \Vert \leq \Vert \x_{0} - \y_{1} \Vert \leq \sqrt{\lambda_{n}^{-1}(\A)}\Vert \x_{0} - \y_{1} \Vert_{\A}$, and since $\x \in \mK$, it holds that $  \Vert \y_{1} \Vert \leq R +  \sqrt{3 \epsilon \lambda_{n}^{-1}(\A)} $, and the lemma holds trivially.

For the remaining of the proof we shall assume that $\Vert \x_{0} - \y_{1} \Vert_{\A}^2 > 3\epsilon$. Let us denote by $k \geq 1$ the overall number of for loop iterations executed in Algorithm \ref{alg:CIP-FW}, i.e., $\Vert \y_{k} - \x_{k} \Vert_{\A}^2 \leq 3 \epsilon$ and $\Vert \y_{i} - \x_{i} \Vert_{\A}^2 > 3 \epsilon$ for all $ i < k $. 
Using Lemma \ref{lemma:SH-FW} we have that for all $i < k$ it holds that $\left( \y_{i} - \z \right)^\top \A \left( \y_{i} - \x_{i} \right) \geq (2/3)\Vert \y_{i} - \x_{i} \Vert_{\A}^2$ for every $\z \in \mK$. 
Thus, using Lemma \ref{lemma:update_step_with_hp} with $\g = \left( \y_{i} - \x_{i} \right) , C= \Vert \y_{i} - \x_{i} \Vert_{\A}, Q=(2/3)\Vert \y_{i} - \x_{i} \Vert_{\A}^2 $, and $\gamma = Q/C^2 = 2/3$, we have that for every $ i < k$,
\begin{align}
    \forall\z\in\mK:\quad \Vert \y_{i+1} -\z \Vert_{\A}^2 \leq \Vert \y_{i} -\z \Vert_{\A}^2 - (4/9) \Vert \y_{i} - \x_{i} \Vert_{\A}^2. \label{eq:update_hyperplane_linear_optimization_oracle}
\end{align}
Thus, we obtain that for all $\z\in\mK$ it holds that $\Vert \y_k - \z \Vert_\A^2 \leq  \Vert \y_{1} - \z \Vert_\A^2$. 

Now we continue to upper-bound the number of iterations until Algorithm \ref{alg:CIP-FW} stops. Denote $\x_{i}^* = \argmin_{\x \in \mK} \Vert \y_{i} - \x \Vert_\A^2$ for every iteration $i < k$. Using Eq. \eqref{eq:update_hyperplane_linear_optimization_oracle}, for every iteration $i < k$ it holds that 
\begin{align*}
    \dist_\A^2 (\y_{i+1}, \mK)  & = \Vert  \y_{i+1} - \x_{i+1}^* \Vert_\A^2 \leq \Vert \y_{i+1} - \x_{i}^* \Vert_\A^2   \\
    & \leq \Vert \y_{i} - \x_{i}^* \Vert_\A^2  - (4/9) \Vert \y_{i} - \x_{i} \Vert_{\A}^2  \leq e^{-(4/9)} \dist_\A^2 (\y_{i}, \mK),
\end{align*}
where the last inequality is since $\dist_{\A}^2 (\y_{i}, \mK) =\Vert{\y_i-\x_i^*}\vert_{\A}^2 \leq \Vert \y_{i} - \x_{i} \Vert_{\A}^2$, and by using the inequality $1-x \leq e^{-x}$. Unrolling the recursion we have,
\begin{align*}
    \dist_\A^2 (\y_{k}, \mK)  & \leq e^{-(4/9)(k-1)}\dist_\A^2 (\y_{1}, \mK)  \leq e^{-(4/9)(k-1)} \Vert \y_{1} - \x_{0} \Vert_\A^2.
\end{align*}
Thus, after at most $k -1 = 2.25\log\left( \Vert \y_{1} -\x_{0} \Vert_{\A}^2 / \epsilon \right)$ iterations, we obtain $\dist_\A^2 (\y_{k}, \mK) \leq \epsilon$, which by using Lemma \ref{lemma:SH-FW}, implies that the next iteration will be the last one, and the points $ \x_k,\y_k$ will indeed satisfy $\Vert \x_k - \y_k \Vert_\A^2 \leq 3\epsilon$. This proves the upper-bound on the overall number of for loop iterations.

Finally, note that $\x\in\mK$ since it is the output of Algorithm \ref{alg:SH-FW}, and since $\Vert \y \Vert - \Vert \x \Vert \leq \Vert \x - \y \Vert \leq \sqrt{\lambda_{n}^{-1}(\A)}\Vert \x - \y \Vert_{\A}$, we obtain $  \Vert \y \Vert \leq R +  \sqrt{3 \epsilon \lambda_{n}^{-1}(\A)} $ as required.
\end{proof}

\section{Full Version of Theorem \ref{thm:mainthm:short} and Proof}

Before we state the full version of Theorem  \ref{thm:mainthm:short} and prove it, we first prove Lemma \ref{lem:LOO-ONS} and an additional lemma.

\begin{proof}[Proof of Lemma \ref{lem:LOO-ONS}]
First, note that since for every block $m$ in Algorithm \ref{alg:ONS-WF} we have that $\A_m \succeq \A_0 = \epsilon_I\I_n$, it holds that  $\lambda_n(\A_m)\geq\epsilon_I$. By our assumption that $\frac{3\epsilon}{\epsilon_I}\leq 4R^2$, using Lemma \ref{lemma:CIP-FW}, it follows that  for every block $m$  in Algorithm \ref{alg:ONS-WF} we have that $\widetilde{\y}_m\in{}3R\ball$, which is in accordance with the assumption of Lemma \ref{lemma:ONS-WF}. Since $\eta \geq \max\{ 12KGR, \frac{2K}{\alpha} \}$, and $\epsilon_I \geq (KG)^2$, the regret bound stated in the lemma follows immediately from the one in  Lemma \ref{lemma:ONS-WF}.

Now we move on to upper-bound the overall number of calls to the LOO. We note that Eq. \eqref{eq:Am_to_Am_minus_one1} holds here as well from the same arguments stated in the proof of Lemma \ref{lemma:ONS-WF}. 
 Recall that the update step of Algorithm \ref{alg:ONS-WF} is $\y_{m+1} = \widetilde{\y}_{m} - \eta \A_{m}^{-1} \bar{\nabla}_{m}$. Thus, by using Eq. \eqref{eq:Am_to_Am_minus_one1}, and the fact that the points  $\x_{m}, \widetilde{\y}_{m}$ are the outputs of $\mathcal{O}_{AFP}$ when called from Algorithm \ref{alg:ONS-WF} with the input point $\y_m$, the positive definite matrix $\A_{m-1}$, and an error tolerance of $3\epsilon$,  we have that
\begin{align*}
    \Vert \x_{m} - \y_{m+1} \Vert_{\A_{m}} \leq \Vert \x_{m} - \widetilde{\y}_{m} \Vert_{\A_{m}} + \Vert \widetilde{\y}_{m} - \y_{m+1} \Vert_{\A_{m}} \leq \sqrt{6 \epsilon} + \eta \Vert \A_{m}^{-1} \bar{\nabla}_{m} \Vert_{\A_{m}}. 
\end{align*}
Using $(a+b)^2 \leq 2a^2 + 2b^2$ and since $\lambda_{1}(\A_m^{-1}) = \lambda_n^{-1}\brac{\A_m} \leq \lambda_n^{-1}\brac{\A_0}= 1/\epsilon_I$, we have that for any block $m$ in Algorithm \ref{alg:ONS-WF},
\begin{align}
    \Vert \x_{m} - \y_{m+1} \Vert_{\A_{m}}^2 \leq 12  \epsilon + 2 \eta^2   \Vert  \bar{\nabla}_{m} \Vert_{\A_{m}^{-1}}^2 \leq 12  \epsilon + 2 \eta^2 K^2 G^2 / \epsilon_I, \label{eq:euc_dist_input_cip}
\end{align}
where the last inequality also uses the fact that $\Vert{\bar{\nabla}_m}\Vert \leq KG$.

Using Lemma \ref{lemma:CIP-FW}, each call to Algorithm \ref{alg:CIP-FW} on some block $m$ of Algorithm \ref{alg:ONS-WF}, makes at most $\max\{2.25\log \brac{ \Vert \x_{m} - \y_{m+1} \Vert_{\A_{m}}^2 / \epsilon }+1, 0\}$ iterations. On each such  iteration  it calls  Algorithm \ref{alg:SH-FW} which, according to Lemma \ref{lemma:SH-FW}, makes at most $\left\lceil \frac{27R^2 \lambda_{1}(\A_m)}{\epsilon } -2 \right\rceil$ calls to the LOO. Recall that by the update rule of Algorithm \ref{alg:ONS-WF}: $\lambda_{1}(\A_m) \leq \lambda_{1}(\A_{m-1}) + \lambda_1(\bar{\nabla}_m\bar{\nabla}_m^{\top}) \leq \dots \leq \epsilon_I + G^2KT$ for every $m$, where we again used the fact that $\Vert{\bar{\nabla}_m}\Vert \leq KG$ for every block $m$. Thus, Algorithm \ref{alg:ONS-WF}, on each block $m$, makes at most 
\begin{align*}
    n_m & \leq  \brac{ 2.25 \log \left(\frac{12  \epsilon + 2 \eta^2 K^2 G^2 / \epsilon_I}{\epsilon} \right)  + 1 }\frac{27 R^2 (\epsilon_I + G^2KT)}{\epsilon } \\
	& = 2.25 \brac{\log \left(\frac{12  \epsilon + 2 \eta^2 K^2 G^2 / \epsilon_I}{\epsilon} \right)  + \log\brac{e^ \frac{4}{9}} }\frac{27 R^2 (\epsilon_I + G^2KT)}{\epsilon } \\
	&\leq  2.25\log \left( 1.56 \frac{12  \epsilon + 2 \eta^2 K^2 G^2 / \epsilon_I}{\epsilon} \right)  \frac{27 R^2 (\epsilon_I + G^2KT)}{\epsilon } 
\end{align*}
calls to the LOO. Thus, the overall number of calls to the LOO throughout the run of Algorithm \ref{alg:ONS-WF} is 
\begin{align*}
    N_{calls}  = \sum_{m=1}^{T/K} n_m & \leq 61 R^2 \log \left( 19  + 4 \frac{ \eta^2 K^2 G^2}{\epsilon  \epsilon_I} \right) \frac{T}{K \epsilon}(\epsilon_I + G^2KT). 
\end{align*}
\end{proof}

\begin{lemma}\label{lemma:bound_sum_Anorm_grad}
Consider Algorithm \ref{alg:ONS-WF} with the (standard ONS) update rule: $\A_m = \A_{m-1} + \bar{\nabla}_m \bar{\nabla}_m^\top$ for every block $m$. Then, for every $\rho\in[n]$ it holds that, 
\begin{align*}
    \sum_{m=1}^{T/K} \matnorm{{\bar{\nabla}}_m}{\A_{m}^{-1}}^2 \leq  \rho \log\brac{ \frac{TKG^2 + \epsilon_I}{\epsilon_I} } + \frac{K}{ \epsilon_I} \sum_{i=\rho+1}^{n} \lambda_i \brac{\sum_{t=1}^{T} {\nabla}_t {\nabla}_t^\top}  .
\end{align*}
\end{lemma}

\begin{proof}
For every $m \geq 1$ it holds that
\begin{align*}
    \matnorm{\bar{\nabla}_m} {\A_{m}^{-1}}^2 = \bar{\nabla}_m^\top \A_{m}^{-1} \bar{\nabla}_m = \A_{m}^{-1} \bullet \bar{\nabla}_m \bar{\nabla}_m^\top = \A_{m}^{-1} \bullet (\A_{m} - \A_{m-1}).
\end{align*}
Since $\A_{m},\A_{m-1} \succ 0$, using Lemma \ref{lemma:matrix_ratio}, for every $m \geq 1$ we have that,
\begin{align*}
    \matnorm{ \bar{\nabla}_m}{\A_{m}^{-1}}^2 \leq \log{\frac{|\A_{m}|}{|\A_{m-1}|}}.
\end{align*} 
Summing over $m \in [T/K]$, we have
\begin{align*}
    \sum_{m=1}^{T/K} \Vert \bar{\nabla}_m \Vert_{\A_{m}^{-1}}^2 \leq \sum_{m=1}^{T/K} \log{\frac{|\A_{m}|}{|\A_{m-1}|}} = \log{\frac{|\A_{T/K}|}{|\A_{0}|}}.
\end{align*}
Since according to the update rule listed in the lemma we have that $\A_{T/K} = \epsilon_I \I_n + \sum_{m=1}^{T/K} \bar{\nabla}_m \bar{\nabla}_m^\top$, and $\enorm{ \bar{\nabla}_m}^2 \leq K^2G^2$ then, $\lambda_{1}(\A_{T/K}) \leq (\epsilon_I + TKG^2 )$. Using Weyl's inequality for the eigenvalues, we have that for every $\rho \in [n]$ it holds that,
\begin{align*}
    |\A_{T/K}| &  = \prod_{i=1}^{\rho}\lambda_i \brac{\epsilon_I \I_n + \sum_{m=1}^{T/K} \bar{\nabla}_m \bar{\nabla}_m^\top}\prod_{i=\rho+1}^{n} \lambda_i \brac{\epsilon_I \I_n + \sum_{m=1}^{T/K} \bar{\nabla}_m \bar{\nabla}_m^\top} \\
    & \leq \brac{TKG^2 + \epsilon_I}^\rho \prod_{i=\rho+1}^{n}  \brac{\epsilon_I + \lambda_i\brac{ \sum_{m=1}^{T/K} \bar{\nabla}_m \bar{\nabla}_m^\top} }.
\end{align*}
Recall that $\A_0 = \epsilon_I \I_n$. Since $\bar{\nabla}_m = \sum_{t=(m-1)K+1}^{mK} {\nabla}_t$,  it holds that $K \sum_{t=1}^{T} {\nabla}_t {\nabla}_t^\top \succeq \sum_{m=1}^{T/K} \bar{\nabla}_m \bar{\nabla}_m^\top$ (Lemma  \ref{lemma:eigenvalues_sum_connection}). Thus, we have that
\begin{align*}
    \frac{|\A_{T/K}|}{|\A_0|} &\leq \brac{1 + \frac{TKG^2 }{\epsilon_I}}^\rho \prod_{i=\rho+1}^{n} \brac{1 +\frac{ \lambda_i\brac{ K \sum_{t=1}^{T} {\nabla}_t {\nabla}_t^\top} }{ \epsilon_I}} \\
    & \leq \brac{1 + \frac{TKG^2}{\epsilon_I}}^\rho  \brac{e^{\brac{K / \epsilon_I} \sum_{i=\rho+1}^{n} \lambda_i\brac{\sum_{t=1}^{T} {\nabla}_t {\nabla}_t^\top}  } },
\end{align*}
where the last inequality follows from using $1+x\leq e^x$. 

Thus, we obtain
\begin{align*}
    \sum_{m=1}^{T/K} \Vert \bar{\nabla}_m \Vert_{\A_{m}^{-1}}^2 \leq \log\brac{\frac{|\A_{T/K}| }{|\A_0|}} \leq \rho \log\brac{1+ \frac{TKG^2 }{\epsilon_I}} + \frac{K}{ \epsilon_I} \sum_{i=\rho+1}^{n} \lambda_i\brac{\sum_{t=1}^{T} {\nabla}_t {\nabla}_t^\top}.
\end{align*}
\end{proof}

\begin{theorem}[Full version of Theorem \ref{thm:mainthm:short}] \label{thm:LOO-ONS-full-version}
Consider the implementation of Algorithm \ref{alg:ONS-WF} as described in Lemma \ref{lem:LOO-ONS} and when using the (standard ONS) update rule: $\A_m = \A_{m-1} + \bar{\nabla}_m \bar{\nabla}_m^\top$ for every block $m$.
\begin{enumerate}
\item
Suppose $T \geq T_0 = c  \log^3 \brac{ c  +  \brac{c+\frac{c}{R^2G^2\alpha^2}} n^{-\frac{4}{3}} T^{\frac{1}{3}}} $, where $c>0$ is a certain universal constant. Setting 
\begin{align*}
&\eta = 8 \max\{ 6GR, \frac{1}{\alpha} \} n^{-\frac{1}{3}} T^\frac{2}{3},  \quad K = 4 n^{-\frac{1}{3}} T^\frac{2}{3}, \quad \epsilon_{I} = 32 G^2  T^\frac{4}{3},\\
&\epsilon= 96G^2 R^2\log \left( 19  + 8 \brac{12+\frac{1}{3R^2G^2\alpha^2}}  n^{-\frac{4}{3}} T^{\frac{1}{3}} \right) T 
\end{align*}
in Algorithm \ref{alg:ONS-WF}, the regret is upper-bounded by
\begin{align*} 
    \sum_{t=1}^{T} f_t(\x^t) - \min_{\x^* \in \mK}  \sum_{t=1}^{T} f_t(\x^*) \leq &   9 \beta    R^2 T^\frac{2}{3} \log \left(  \brac{c+\frac{c}{R^2G^2\alpha^2}}  T^{\frac{1}{3}} \right)  + 2 R G n^{\frac{1}{3}}  T^\frac{2}{3} \\
    & +  \brac{ 36GR+ \frac{4}{\alpha} }  n^{\frac{2}{3}} T^{\frac{2}{3}}  \log \left(   \brac{c+\frac{c}{R^2G^2\alpha^2}} T^{\frac{1}{3}} \right). 
\end{align*}
\item
In continue to the previous item and under the same choice of parameters, for any $\rho\in[n]$, denoting $\Omega_{\rho} = \sum_{i=\rho+1}^n\lambda_i(\sum_{t=1}^T\nabla_t\nabla_t^{\top})$ ($\nabla_t$ is as defined in Algorithm \ref{alg:ONS-WF}), the regret is upper-bounded by
\begin{align*} 
    \sum_{t=1}^{T} f_t(\x^t) - \min_{\x^* \in \mK} \sum_{t=1}^{T} f_t(\x^*) \leq & 9 \beta  R^2  T^\frac{2}{3} \log \left( \brac{c+\frac{c}{R^2G^2\alpha^2}}  T^{\frac{1}{3}} \right) + 2 R G n^{\frac{1}{3}} T^\frac{2}{3}  \\
    & + 36GR \rho^\frac{1}{2} n^{\frac{1}{6}} T^\frac{2}{3} \log \left(  \brac{c+\frac{c}{^2G^2\alpha^2}}  T^{\frac{1}{3}} \right) \\
    & +  \frac{4}{\alpha} \rho n^{-\frac{1}{3}} T^\frac{2}{3} \log \left( c  T^{\frac{1}{3}} \right)  \\
    & + 5  R  T^{\frac{1}{3}} \sqrt{\Omega_{\rho}  } \sqrt{ \log \left(  \brac{c+\frac{c}{R^2G^2\alpha^2}} T^{\frac{1}{3}} \right) } \\
    & + \frac{  \brac{ 3GR + \frac{1}{2\alpha} }  }{ G^2 n^{\frac{2}{3}} } \Omega_{\rho}. 
\end{align*}
\item
Fix $\rho\in[n]$. Suppose $T \geq T_0 = c  \log^3 \brac{ c  +  \brac{c+\frac{c}{R^2G^2\alpha^2}} \rho^{-\frac{4}{3}} T^{\frac{1}{3}}} $, where $c$ is as in the previous items. Setting 
\begin{align*}
&\eta = 8 \max\{ 6GR, \frac{1}{\alpha} \} \rho^{-\frac{1}{3}} T^\frac{2}{3}, \quad K = 4 \rho^{-\frac{1}{3}} T^\frac{2}{3}, \quad \epsilon_{I} = 32 G^2 T^\frac{4}{3},\\
&\epsilon= 96 G^2 R^2 \log \left( 19  + 8 \brac{12+\frac{1}{3R^2G^2\alpha^2}} \rho^{-\frac{4}{3}} T^{\frac{1}{3}} \right) T 
\end{align*}
 in Algorithm \ref{alg:ONS-WF}, the regret is upper-bounded by
\begin{align*} 
    \sum_{t=1}^{T} f_t(\x^t) - \min_{\x^* \in \mK} \sum_{t=1}^{T} f_t(\x^*) \leq & 9 \beta R^2 T^\frac{2}{3} \log \left( \brac{c+\frac{c}{R^2G^2\alpha^2}} T^{\frac{1}{3}} \right) + 2RG\rho^\frac{1}{3}T^\frac{2}{3} \\
    & + \brac{ 36GR + \frac{4}{\alpha} } \rho^{\frac{2}{3}} T^\frac{2}{3}  \log \left( \brac{c+\frac{c}{R^2G^2\alpha^2}} T^{\frac{1}{3}} \right)  \\
    & + 5 R T^\frac{1}{3} \sqrt{ \Omega_{\rho} } \sqrt{  \log \left( \brac{c+\frac{c}{R^2G^2\alpha^2}} T^{\frac{1}{3}} \right)} \\
    &  + \frac{  \brac{ 3GR + \frac{1}{2\alpha} } }{  \rho^{\frac{2}{3}}  G^2 } \Omega_{\rho} . 
\end{align*}
Note this bound is not explicitly dependent on the ambient dimension $n$.
\end{enumerate}
In all cases, the overall number of calls to the LOO is upper-bounded by $0.65  \brac{8 n^\frac{1}{3} T^\frac{2}{3} +T }$, the additional space requirement in $O(n^2)$, and using the Sherman-Morrison formula for fast matrix inversion, the overall additional runtime is $O(n^2(T+n^{1/3}T^{2/3}))$.
\end{theorem}

\begin{proof}
    The regret bound for each of the three cases is obtained directly by combining Lemma \ref{lem:LOO-ONS} with Lemma \ref{lemma:bound_sum_Anorm_grad}, and plugging-in the values of the parameters listed in the theorem. The bound on the overall number of calls follows also from the bound in Lemma \ref{lem:LOO-ONS} and plugging-in the values of the parameters listed in the theorem. The space requirement is dominated by the storage of $\A_m,\A_m^{-1}$ on each block $m$ of Algorithm \ref{alg:ONS-WF}, and is thus upper-bounded by $O(n^2)$.  Finally, in terms of additional runtime, it can be seen that the  most expensive arithmetic operation preformed is the multiplication of a $n\times n$ matrix ($\A_m$ or $\A_m^{-1}$ for some block $m$ of Algorithm \ref{alg:ONS-WF}) with some vector (including when updating $\A_m^{-1}$ from $\A_{m-1}^{-1}$ via the Sherman-Morrison formula for rank-one update of the inverse), which  requires $O(n^2)$ time. It can thus be seen that the overall additional runtime is dominated by the overall number of calls to the LOO (i.e., the overall number of iterations executed by Algorithm \ref{alg:SH-FW} throughout the run of Algorithm \ref{alg:ONS-WF}) times $O(n^2)$, which by plugging-in the values of the parameters in the theorem, gives the listed upper-bound on the overall additional runtime.
\end{proof}

\section{Missing Details from Section \ref{sec:sketch}}
In this section we provide additional details on the sketching algorithm, and give the full version of Theorem \ref{thm:LOO-ONS-FDS} and its proof.

\subsection{Properties of the  sketching algorithm}


The following observation shows that Algorithm \ref{alg:FD-S-ONS} produces matrices $\A_m$, $m=0,\dots,T/K$, that indeed satisfy the requirements of Algorithm \ref{alg:ONS-WF}.
\begin{observation}\label{obs:FD-S-for-ONS}
Fix iteration $m\geq 1$ of Algorithm \ref{alg:FD-S-ONS}. It holds that $\A_0 \preceq \A_m \preceq \A_{m-1}  + \bar{\nabla}_m \bar{\nabla}_m^\top$.
\end{observation}

\begin{proof}
The first inequality $\A_m \succeq \A_0$ holds trivially due to the definition of $\A_m$ in the algorithm. We thus focus on the proof of the second inequality.
Let $\S_{m-1}$ be as at the beginning of the $m$-th iteration of the for loop, and let $\S_{m-1}^+$ denote its value after setting its last row to $\bar{\nabla}_m$ (instead of $\vz$). It holds that $\S_{m-1}^{+\top} \S_{m-1}^+ = \S_{m-1}^\top \S_{m-1} +  \bar{\nabla}_{m}\bar{\nabla}_{m}^\top$. By the definition of $\S_m$ in the algorithm we have that,
\begin{align*}
        \S_m^\top \S_m &= \S_{m-1}^{+\top} \S_{m-1}^ + - \sigma_m \V_m^\top \V_m
        = \S_{m-1}^\top \S_{m-1} +  \bar{\nabla}_{m}\bar{\nabla}_{m}^\top  - \sigma_m \V_m^\top \V_m.
    \end{align*}
Since $\A_{m-1} = \epsilon_I\I_n + \S_{m-1}^\top \S_{m-1}$ and $\A_m = \epsilon_I\I_n + \S_m^\top \S_m$,  it indeed holds that
\begin{align}
    \A_m - \A_{m-1} &= \S_m^\top \S_m- \S_{m-1}^\top \S_{m-1} =  \bar{\nabla}_{m}\bar{\nabla}_{m}^\top  - \sigma_m \V_m^\top \V_m  \label{eq:mat_iteration_connection}\\
     & \preceq \bar{\nabla}_m \bar{\nabla}_m^\top \nonumber. 
\end{align}
\end{proof}

We now state several results regarding Algorithm \ref{alg:FD-S-ONS} which will be required in order to prove Theorem \ref{thm:LOO-ONS-FDS}.
\begin{theorem}[Theorem 1.1 in \cite{ghashami2016frequent}] \label{thm:FD-S}
   Consider Algorithm \ref{alg:FD-S-ONS} and let $\B \in \reals^{\brac{T/K} \times n}$  be the matrix which is received by the algorithm  row by row (i.e., the $i$th row of $\B$ is $\bar{\nabla}_i$). It holds that
    \begin{align*}
         \vz \preceq \B^\top \B - \S_{T/K}^\top \S_{T/K} \preceq  \I_n \sum_{i=\rho+1}^{n} \lambda_i \brac{\B^\top \B} .
    \end{align*}
\end{theorem}

\begin{lemma}\label{lemma:bound_sum_sigma_FD}
Consider Algorithm \ref{alg:FD-S-ONS}  and let  $\B \in \reals^{\brac{T/K} \times n}$ be the matrix which is received by the algorithm  row by row (i.e., the $i$th row of $\B$ is $\bar{\nabla}_i$). Denote $\B_\rho$ its best rank-$\rho$ approximation, i.e. $\B_\rho = \argmin_{\C:\rank(\C)\leq \rho}\matnorm{\B - \C}{F}$. It holds that
\begin{align*}
    \sum_{m=1}^{T/K} \sigma_m \leq \matnorm{\B - \B_\rho}{F}^2,
\end{align*}
where $\sigma_m$ is as defined in Algorithm \ref{alg:FD-S-ONS}.
\end{lemma}
\begin{proof}
Let $\z_i$, $i=1,\dots,n$ denote the left singular vector of $\B$ corresponding to the singular value $\sigma_i\brac{\B}$. We denote $\S = \S_{T/K}$, where $\S_{T/K}$ is as defined in Algorithm \ref{alg:FD-S-ONS}, i.e., the last sketched matrix.
Using Property 3 in \cite{ghashami2016frequent} we have that,
\begin{align*}
    (\rho+1) \sum_{m=1}^{T/K} \sigma_m & \leq \matnorm{\B}{F}^2 - \matnorm{\S}{F}^2 = \sum_{i=1}^{n} \enorm{\B \z_i}^2 - \matnorm{\S}{F}^2 =   \sum_{i=1}^{\rho} \enorm{\B \z_i}^2 + \sum_{i=\rho+1}^{n} \enorm{\B \z_i}^2   - \matnorm{\S}{F}^2.
\end{align*}
Since $\B_\rho$ denotes the best rank-$\rho$ approximation of $\B$, the above inequality implies that,
\begin{align}\label{eq:sketchLemma1:1}
    (\rho+1) \sum_{m=1}^{T/K} \sigma_m & \leq    \matnorm{\B - \B_\rho}{F}^2 + \sum_{i=1}^{\rho} \enorm{\B \z_i}^2 - \matnorm{\S}{F}^2. 
\end{align}
Note that
\begin{align*}
\sum_{i=1}^{\rho}\Vert{\S\z_i}\Vert^2 = \sum_{i=1}^{\rho}\z_i^{\top}\S^{\top}\S\z_i \leq \trace(\S^{\top}\S) = \Vert{\S}\Vert_F^2.
\end{align*}
Thus, using Property 2 in \cite{ghashami2016frequent} we have that,
\begin{align}\label{eq:sketchLemma1:2}
    \sum_{i=1}^{\rho} \enorm{\B \z_i}^2 - \matnorm{\S}{F}^2 \leq \sum_{i=1}^{\rho} \brac{\enorm{\B \z_i}^2 - \enorm{\S \z_i}^2 }  \leq  \rho{}\sum_{m=1}^{T/K} \sigma_m
\end{align}
The lemma  follows from plugging Eq. \eqref{eq:sketchLemma1:2} into Eq. \eqref{eq:sketchLemma1:1}.
\end{proof}

\begin{lemma}\label{lemma:sketching_sum_norm_grad}
Consider the run of Algorithm \ref{alg:FD-S-ONS} with a sketch size  $\rho$, and denote $\Omega_\rho = \sum_{i=\rho+1}^{n} \lambda_i \brac{\sum_{t=1}^{T} {\nabla}_t {\nabla}_t^\top}$, where $\nabla_t$ is as defined in Algorithm \ref{alg:ONS-WF}.
It holds that,
\begin{align*}
    \sum_{m=1}^{T/K} \matnorm{ \bar{\nabla}_m }{\A_{m}^{-1}}^2 & \leq \rho\log\brac{ 1 + \frac{ G^2 K T}{\epsilon_I}} + \frac{(\rho+1) K \Omega_\rho}{\epsilon_I }.
\end{align*}
\end{lemma}

\begin{proof}
Using Eq. \eqref{eq:mat_iteration_connection} it holds that,
\begin{align*}
    \sum_{m=1}^{T/K} \matnorm{ \bar{\nabla}_m }{\A_{m}^{-1}}^2 & = \sum_{m=1}^{T/K} \A_{m}^{-1} \bullet \bar{\nabla}_m \bar{\nabla}_m^\top = \sum_{m=1}^{T/K}   \A_{m}^{-1} \bullet  \brac{\A_m - \A_{m-1}} + \sum_{m=1}^{T/K} \sigma_m \A_{m}^{-1} \bullet  \V_m^\top \V_m. 
\end{align*}
Since $\lambda_{1}\brac{\A_m^{-1}} \leq \epsilon_I^{-1}$ and $\trace\brac{\V_m^\top \V_m}  = \rho+1$, it holds that $ \A_{m}^{-1} \bullet  \V_m^\top \V_m \leq \epsilon_I^{-1} (\rho+1)$ for every $m \in [T/K]$. Since $\A_m\succ 0$ for every $m$, using Lemma \ref{lemma:matrix_ratio} we have that,
\begin{align}
    \sum_{m=1}^{T/K} \matnorm{ \bar{\nabla}_m }{\A_{m}^{-1}}^2 & \leq \sum_{m=1}^{T/K} \log{\brac{\frac{|\A_{m}|}{|\A_{m-1}|}}} + \frac{(\rho+1)}{\epsilon_I} \sum_{m=1}^{T/K} \sigma_m \nonumber \\
    &= \log{\brac{\frac{|\A_{T/K}|}{|\A_{0}|}}} + \frac{(\rho+1)}{\epsilon_I} \sum_{m=1}^{T/K} \sigma_m.      \label{eq:sketching_gradients_norm_bound_1}
\end{align}
Since $\rank(\S_{T/K}) \leq \rho$ (recall the last row is $\vz$), we have that
\begin{align*}
    |\A_{T/K}| & = \prod_{i=1}^{n} \lambda_i\brac{\epsilon_I \I_n + \S_{T/K}^\top \S_{T/K}} = \epsilon_I^{\brac{n-\rho}} \prod_{i=1}^{\rho} \brac{ \epsilon_I + \lambda_i \brac{\S_{T/K}^\top \S_{T/K}}}. 
\end{align*}
Since $\A_0 = \epsilon_I \I_n$ and $\S_{T/K}^\top \S_{T/K} \preceq \sum_{m=1}^{T/K} \bar{\nabla}_m \bar{\nabla}_m^\top$ (Theorem \ref{thm:FD-S}), we have that
\begin{align} \label{eq:sketching_gradients_norm_bound_2}
    \log\brac{\frac{|\A_{T/K}| }{|\A_0|}} \leq \sum_{i=1}^{\rho} \log\brac{ 1 + \frac{\lambda_i \brac{\sum_{m=1}^{T/K} \bar{\nabla}_m \bar{\nabla}_m^\top }}{\epsilon_I}}  \leq \rho \log\brac{ 1 + \frac{ G^2 K T}{\epsilon_I}},
\end{align}
where the last inequality is since $\lambda_i \brac{\sum_{m=1}^{T/K} \bar{\nabla}_m \bar{\nabla}_m^\top } \leq  \sum_{m=1}^{T/K}\enorm{\bar{\nabla}_m}^2 \leq (T/K)(KG)^2 = G^2KT$.

From Lemma \ref{lemma:bound_sum_sigma_FD} we have that $\sum_{m=1}^{T/K} \sigma_m \leq \sum_{i=\rho+1}^{n} \lambda_i \brac{ \sum_{m=1}^{T/K} \bar{\nabla}_m \bar{\nabla}_m^\top} $. Since $\bar{\nabla}_m = \sum_{t=(m-1)K+1}^{mK} {\nabla}_t$, using Lemma \ref{lemma:eigenvalues_sum_connection} it holds that $K \sum_{t=1}^{T} {\nabla}_t {\nabla}_t^\top \succeq \sum_{m=1}^{T/K} \bar{\nabla}_m \bar{\nabla}_m^\top$. Plugging these observations and Eq. \eqref{eq:sketching_gradients_norm_bound_2} into Eq.\eqref{eq:sketching_gradients_norm_bound_1}, we indeed obtain
\begin{align*}
    \sum_{m=1}^{T/K} \matnorm{ \bar{\nabla}_m }{\A_{m}^{-1}}^2 & \leq \rho \log\brac{ 1 + \frac{ G^2 K T}{\epsilon_I}} + \frac{(\rho+1) K }{\epsilon_I }\sum_{i=\rho+1}^{n} \lambda_i \brac{\sum_{t=1}^{T} {\nabla}_t {\nabla}_t^\top }. 
\end{align*}
\end{proof}

\subsection{Full version of Theorem \ref{thm:LOO-ONS-FDS} and its proof}
\begin{theorem}\label{thm:BSON-known-r}[Full version of Theorem \ref{thm:LOO-ONS-FDS}]
Consider the implementation of Algorithm \ref{alg:ONS-WF} as described in 
Lemma \ref{lem:LOO-ONS} and when using the update rule described in Algorithm \ref{alg:FD-S-ONS}: $\A_m = \A_0 + \S_m^\top \S_m$ for every block $m$. Fix $\rho \in [n]$ and denote $\Omega_{\rho} = \sum_{i=\rho+1}^{n} \lambda_i \brac{\sum_{t=1}^{T} {\nabla}_t {\nabla}_t^\top}$. Suppose $T \geq T_0 =  c  \log^3 \brac{ c  +  \brac{c+\frac{c}{R^2G^2\alpha}} \rho^{-\frac{4}{3}}T^{\frac{1}{3}}} $, where $c>0$ is a certain universal constant. Setting
\begin{align*}
&\eta = 8 \max\{ 6GR, \frac{1}{\alpha} \} \rho^{-\frac{1}{3}} T^\frac{2}{3},  \quad K = 4 \rho^{-\frac{1}{3}} T^\frac{2}{3}, \quad \epsilon_{I} =  32 G^2  T^\frac{4}{3},\\
&\epsilon= 96G^2 R^2\log \left( 19  + 8 \brac{12+\frac{1}{3R^2G^2\alpha^2}} \rho^{-\frac{4}{3}} T^{\frac{1}{3}} \right) T 
\end{align*}
in Algorithm \ref{alg:ONS-WF}, the regret is upper bounded by 
\begin{align*}
    \sum_{t=1}^{T} f_t(\x^t) - \min_{\x^* \in \mK} \sum_{t=1}^{T} f_t(\x^*)  \leq & 9 \beta R^2 T^\frac{2}{3} \log \left(   \brac{c+\frac{c}{R^2G^2\alpha^2}} T^{\frac{1}{3}} \right) + 2 R G \rho^{\frac{1}{3}} T^\frac{2}{3} \\
    & + \brac{ 36GR+ \frac{4}{\alpha} } \brac{\rho+1}^{\frac{2}{3}} T^\frac{2}{3}  \log \left(  \brac{c+\frac{c}{R^2G^2\alpha^2}} T^{\frac{1}{3}} \right)   \\
    & + 5 R \rho^\frac{1}{2} T^\frac{1}{3}   \sqrt{ \Omega_{\rho} } \sqrt{\log \left(  \brac{c+\frac{c}{R^2G^2\alpha^2}} T^{\frac{1}{3}} \right) }\\
    &  +  \brac{ 6GR + \frac{1}{\alpha} }  \frac{ \rho^{\frac{1}{3}} \Omega_{\rho} }{ G^2   }.
\end{align*}
The overall number of calls to the LOO is upper bounded by $0.65  \brac{8 \rho^{\frac{1}{3}} T^\frac{2}{3}  + T}$, the additional space requirement in $O(\rho n)$, and the overall additional runtime is $O(\rho n T + \rho^{4/3}nT^{2/3} + \rho^{7/3}nT^{1/3} )$.
\end{theorem}

\begin{proof}
    The upper-bound on the regret and number of calls to the LOO follows directly from combining Lemma \ref{lem:LOO-ONS}, Lemma \ref{lemma:sketching_sum_norm_grad}, and plugging-in the values for parameters listed in the theorem.  In terms of space requirement, by only explicitly maintaining the  $(\rho+1)\times n$ matrices $\S_m,\V_m$ and the diagonal matrices $\H_m,\Sigma_m,\widehat{\Sigma}_m$ on each block $m$ of Algorithm \ref{alg:ONS-WF}, i.e., $\A_m,\A_m^{-1}$ are never computed explicitly, the space is upper-bounded by $O(\rho{}n)$. Finally, for upper-bounding the overall additional runtime we note that the two most expensive arithmetic operations are i. computing a matrix-vector product with either $\A_m$ or $\A_{m}^{-1}$ during some block $m$ of Algorithm \ref{alg:ONS-WF}, and ii. computing the eigen-decomposition of $\S_{m-1}^{\top}\S_{m-1}\in\reals^{n\times n}$ for some iteration $m$ of Algorithm \ref{alg:FD-S-ONS}. Using the low-rank factorizations of $\A_m,\A_m^{-1}$ in Algorithm \ref{alg:FD-S-ONS} (i.e., by only explicitly maintaining $\S_m\in\reals^{(\rho+1)\times n}$ and the diagonal matrix $\H_m$), computing a matrix-vector product with either $\A_m$ or $\A_m^{-1}$, could be carried out in $O(\rho{}n)$ time. Computing the eigen-decomposition of each $\S_{m-1}^{\top}\S_{m-1}$, could be done in $O(\rho^2{}n)$ time by computing the SVD of $\S_{m-1}\in\reals^{(\rho+1)\times n}$. Note however that such an SVD is computed only once during each block $m$ of Algorithm \ref{alg:ONS-WF}. Thus, the overall runtime associated with these SVD computations is $O((T/K)\rho^2n)$, which by plugging the value of $K$ in the theorem, is only $O(\rho^{7/3}nT^{1/3})$. The overall additional runtime, excluding these SVD computations, is thus dominated by the number of matrix-vector products, times the runtime required for each such product which, as discussed, is $O(\rho{}n)$. As discussed in the proof of Theorem \ref{thm:LOO-ONS-full-version}, the number of matrix-vector products is dominated by the overall number of calls to the LOO (i.e., the overall number of iterations executed by Algorithm \ref{alg:SH-FW}). Combining these two contributions (matrix-vector products and SVD computations) to the overall additional runtime, yields the bound listed in the theorem.

\end{proof}

\section{Discussion of Assumption \ref{ass:mainass}}\label{sec:AssDiscuss}
We recall that while the feasible set is $\mK$, Assumption \ref{ass:mainass} assumes the losses are defined and satisfy the various properties listed in the (potentially much) larger set $3R\ball$, where $R$ is such that $\mK\subseteq{}R\ball$. This is because our Algorithm \ref{alg:ONS-WF} queries gradients at infeasible points w.r.t. $\mK$, and thus we must make sure these assumptions hold in these infeasible points, and for ease of presentation we simply make sure in our analysis that indeed all infeasible points $\widetilde{\y}_m$ in the instantiations of  Algorithm \ref{alg:ONS-WF}, lie inside the ball $3R\ball$. 

First, we note that our consideration of enclosing balls centered at the origin is w.l.o.g. since one can apply translation. Second, with a slightly more involved analysis it suffices to require that Assumption \ref{ass:mainass} holds only in the set $\widetilde{\mK}_{\delta} := \{\x\in\reals^n~|~\dist(\x,\mK)\leq \delta\}$, for some small $\delta >0$, as we now explain. Note in particular that Lemma \ref{lemma:CIP-FW} guarantees that our LOO-based implementation of the AFP oracle (Algorithm \ref{alg:CIP-FW}) returns an infeasible point $\y$ and a corresponding feasible point $\x\in\mK$, such that $\Vert{\x-\y}\Vert_{\A}^2 \leq 3\epsilon$. This implies that $\dist(\y,\mK) \leq \Vert{\x-\y}\Vert \leq \sqrt{\frac{3\epsilon}{\lambda_n(\A)}}$. Thus, when used with our Algorithm \ref{alg:ONS-WF} (as described in Lemma \ref{lem:LOO-ONS}), for every block $m$ we have that,
\begin{align*}
\dist(\widetilde{\y}_m,\mK) \leq  \sqrt{\frac{3\epsilon}{\lambda_n(\A_{m-1})}} \leq \sqrt{\frac{3\epsilon}{\lambda_n(\A_0)}}
=\sqrt{\frac{3\epsilon}{\epsilon_I}} =\widetilde{O}(T^{-1/6}),
\end{align*}
where the second inequality is due to the constraints on the matrices $\{\A_m\}_{m\geq 1}$ in Algorithm \ref{alg:ONS-WF}, 
and the last equality follows from plugging-in the values of $\epsilon,\epsilon_I$ listed in our main theorems --- Theorem \ref{thm:LOO-ONS-full-version} and Theorem \ref{thm:LOO-ONS-FDS}.

Thus, already for $\delta = \widetilde{O}(T^{-1/6})$, all the points in which our algorithm queries gradients lie in $\widetilde{\mK}_{\delta}$, and it suffices to require that the assumptions listed in Assumption \ref{ass:mainass} hold only in this set, which becomes tighter around $\mK$ as $T$ increases.

\section{Auxiliary Lemmas}
\begin{lemma}\label{lemma:eigenvalues_sum_connection}
Let $\{\v_i\}_{i=1}^{k}\subset\reals^n $, and $\u = \sum_{i=1}^k \v_i$. Then, 
    $k \sum_{i=1}^{k} \v_i \v_i^\top \succeq  \u \u^\top$.

\end{lemma}
\begin{proof}
From Jensen's inequality we have that for every sequence of scalars $\{a_i\}_{i=1}^k\subset\reals$, it holds that $k \sum_{i=1}^{k} a_i^2 \geq \left( \sum_{i=1}^{k} a_i \right)^2$. Thus, for every $\z \in \reals^n$ we have that,
\begin{align*}
    \z^\top  \brac{ \sum_{i=1}^{k} \v_i \v_i^\top} \z & = \sum_{i=1}^{k} \brac{\v_i^\top \z }^2  \geq \frac{1}{k} \brac{ \sum_{i=1}^{k} \v_i^\top \z }^2  = \frac{1}{k} \brac{\u^\top \z }^2  = \frac{1}{k} \z^\top \u \u^\top \z. 
\end{align*}
\end{proof} 

\begin{lemma}\label{lemma:block_property}
Let $\mC\subset\reals^n$ be convex and compact and such that $\mC\subseteq{}R\ball$.
Let $f_1,\dots,f_k$ be functions $\mC\rightarrow\reals$ that are differentiable over $\mC$ and have gradients upper-bounded in $\ell_2$ norm by some $G>0$ over $\mC$, and satisfy the curvature condition (Definition \ref{def:exp_concave_property}) over $\mC$ with some parameter $\alpha > 0$. Define $h(\x) = \sum_{i=1}^kf_i(\x)$. 
    For all $\eta \geq \max\{ 4kGR, 2k/\alpha \}$ and every $\x,\y \in \mC$ it holds that,
    \begin{align*}
        h(\x) - h(\y)  \leq \nabla h(\x)^\top \brac{\x -\y} -  \frac{1}{2 \eta} \brac{\y -\x}^\top \nabla h(\x) \nabla h(\x)^\top \brac{\y -\x}
    \end{align*}
\end{lemma}

\begin{proof}
    Let $\eta' = \max\{ 4GR, 2/\alpha \}$ and fix some $\x,\y\in\mC$. It holds that
    \begin{align*}
        h(\x) - h(\y) & = \sum_{i=1}^k f_i(\x) - f_i(\y) \\
        &\leq  \sum_{i=1}^k \brac{\nabla f_i(\x)^\top \brac{\x -\y}} - \frac{1}{2 \eta'} \sum_{i=1}^k \brac{ \brac{\y -\x}^\top \nabla f_i(\x) \nabla f_i(\x)^\top \brac{\y -\x}} \\
        & = \nabla{}h(\x)^\top \brac{\x -\y} - \frac{1}{2 \eta'}  \brac{\y -\x}^\top \brac{\sum_{i=1}^k \nabla f_i(\x) \nabla f_i(\x)^\top }\brac{\y -\x}.
    \end{align*}
    Using Lemma \ref{lemma:eigenvalues_sum_connection} we have that $k \sum_{i=1}^k \nabla f_i(\x) \nabla f_i(\x)^\top \succeq \brac{\sum_{i=1}^k \nabla f_i(\x)} \brac{\sum_{i=1}^k \nabla f_i(\x) }^\top$ and thus, 
    \begin{align*}
        h(\x) - h(\y) & \leq \nabla  h(\x)^\top \brac{\x -\y} - \frac{1}{2 \eta' k}  \brac{\y -\x}^\top \brac{\brac{\sum_{i=1}^k \nabla f_i(\x)} \brac{\sum_{i=1}^k \nabla f_i(\x) }^\top}\brac{\y -\x}\\
        & \leq \nabla  h(\x)^\top \brac{\x -\y} - \frac{1}{2 \eta}  \brac{\y -\x}^\top \nabla h(\x) \nabla h(\x) ^\top \brac{\y -\x},
    \end{align*}
    where the last inequality holds since $\eta \geq k \eta'$.
\end{proof}

\begin{lemma}
\label{lemma:matrix_ratio}
Let $\A ,\B \in \mbS^{n}$ be positive definite matrices. Then, 
    $\A^{-1} \bullet ( \A - \B ) \leq \ln \frac{ |\A| }{ |\B| }$.
\end{lemma}
\begin{proof}
It holds that,
\begin{align*}
    \A^{-1} \bullet \brac{ \A - \B } & = \trace\brac{\A^{-\frac{1}{2}}\brac{\A - \B }\A^{-\frac{1}{2}}} = \sum_{i=1}^{n} \brac{ \lambda_i\brac{\I_n - \A^{-\frac{1}{2}}\B \A^{-\frac{1}{2}}} }.
\end{align*}
Since $\lambda_i\brac{\I_n-\A} = 1-\lambda_{n-i+1}\brac{\A} $,  $1-x \leq -\ln{x}$ for every $x \in \reals_+$, and $\A^{-\frac{1}{2}}\B \A^{-\frac{1}{2}} \succ \vz$, it holds that 
\begin{align*}
    \A^{-1} \bullet \brac{ \A - \B } & = \sum_{i=1}^{n} \brac{ 1 - \lambda_i\brac{\A^{-\frac{1}{2}}\B \A^{-\frac{1}{2}}} } \leq - \sum_{i=1}^{n} \ln{ \brac{ \lambda_i\brac{\A^{-\frac{1}{2}}\B \A^{-\frac{1}{2}}} } } \\
    & =  -\ln{\brac{\prod_{i=1}^{n} \lambda_i\brac{\A^{-\frac{1}{2}}\B \A^{-\frac{1}{2}}} } } = -\ln{|\A^{-\frac{1}{2}}\B \A^{-\frac{1}{2}} | } \\
    &=\ln{\brac{|\A|/|\B|}}.
\end{align*}
\end{proof}

\end{document}